\documentclass{article}

\usepackage[preprint]{neurips_2023}

\usepackage[utf8]{inputenc} 
\usepackage[T1]{fontenc}    
\usepackage{hyperref}       
\usepackage{url}            
\usepackage{booktabs}       
\usepackage{amsfonts}       
\usepackage{nicefrac}       
\usepackage{microtype}      
\usepackage{xcolor}         
\usepackage{natbib}
\usepackage{amsmath}
\usepackage{bbding}
\usepackage{graphicx,float}
\usepackage{threeparttable}
\usepackage{comment}

\usepackage{array}
\newcolumntype{L}[1]{>{\raggedright\let\newline\\\arraybackslash\hspace{0pt}}m{#1}}
\newcolumntype{C}[1]{>{\centering\let\newline\\\arraybackslash\hspace{0pt}}m{#1}}
\newcolumntype{R}[1]{>{\raggedleft\let\newline\\\arraybackslash\hspace{0pt}}m{#1}}

\usepackage{amsthm}
\newtheorem{definition}{Definition}
\newtheorem{example}{Example}
\newtheorem{theorem}{Theorem}
\newtheorem{remark}{Remark}
\newtheorem{lemma}{Lemma}
\newtheorem{proposition}{Proposition}

\title{Exploring the Complexity of Deep Neural Networks through Functional Equivalence}

\author{%
	Guohao Shen 
	\\
	Department of Applied Mathematics\\
	The Hong Kong Polytechnic University\\
	Hung Hom, Kowloon, Hong Kong SAR, China\\
	\texttt{guohao.shen@polyu.edu.hk} \\
}

\begin{document}

	\maketitle

	\begin{abstract}
	We investigate the complexity of deep neural networks through the lens of functional equivalence, which posits that different parameterizations can yield the same network function. Leveraging the equivalence property, we present a novel bound on the covering number for deep neural networks, which reveals that the complexity of neural networks can be reduced. Additionally, we demonstrate that functional equivalence benefits optimization, as overparameterized networks tend to be easier to train since increasing network width leads to a diminishing volume of the effective parameter space. These findings can offer valuable insights into the phenomenon of overparameterization and have implications for understanding generalization and optimization in deep learning.
	\end{abstract}

	\section{Introduction}

Artificial neural networks, particularly deep and wide ones, have shown remarkable success in various applications widely in machine learning and artificial intelligence. However, one of the major challenges in understanding the success is to explain their generalization ability when they are very large and prone to overfitting data \citep{neyshabur2014search,neyshabur2017exploring,razin2020implicit}. 

Theoretical studies have suggested that the generalization error can be related to the complexity, approximation power, and optimization of deep neural networks. Larger neural networks are proved to possess better approximation power \citep{yarotsky2017error,lu2020deep,zhou2020universality}, but may exhibit larger complexity and generalization gaps \citep{bartlett2017spectrally,mohri2018foundations,bartlett2019nearly}, and can be more challenging to optimize \citep{glorot2011deep}. However, some aspects of deep learning initially appeared to contradict common sense: overparameterized networks tend to be easier to train \citep{frankle2018the,allen2019convergence,du2019gradient} and exhibit better generalization \citep{belkin2019reconciling,neyshabur2018the,novak2018sensitivity}. Although the model class's capacity was immense, deep networks did not tend to overfit \citep{zhang2017understanding}.

Recent studies have highlighted that the functional form of neural networks may be less complex than their parametric form \cite{bui2020functional,stock2022embedding,grigsby2022functional}, as networks with different parameters may implement the same function. This insight provides us with a fresh perspective for reconsidering how overparameterization truly affects the generalization.

In this work, we quantitatively characterize the redundancy in the parameterization of deep neural networks and derive a complexity measure for these networks based on functional equivalence. We analyze the results to gain insights into generalization and optimization in deep learning.

\subsection{Related work}
The issue of redundancy or identification of parameterization of neural networks has been noted since 1990 in \citet{hecht1990algebraic}. Subsequent studies for neural networks with Tanh and sigmoid activations \citep{chen1993geometry,fefferman1993recovering,kuurkova1994functionally} have proved that given the input-output mapping of a Tanh neural network, its architecture can be determined and weights are identified up to permutations and sign flips. Recently, the identifiability of parameterization in deep neural networks, particularly ReLU networks, has received considerable attention \citep{elbrachter2019degenerate,bui2020functional,bona2021parameter,dereich2022minimal,stock2022embedding,grigsby2022functional,grigsby2022local}. Most recently, \cite{bui2020functional} demonstrated that ReLU networks with non-increasing widths are identifiable up to permutation and scaling of weight matrices. With redundant parameterization, the weight space of deep neural networks can exhibit symmetric structures, which leads to implications for optimization \citep{neyshabur2015path,badrinarayanan2015symmetry,stock2018equinormalization}. These studies suggest that naive loss gradient is sensitive to reparameterization by scaling, and proposed alternative, scaling-invariant optimization procedures. In addition, the redundancy or identification properties of neural networks are closely related to the study of inverse stability \citep{elbrachter2019degenerate,rolnick2020reverse,bona2021parameter,petersen2021topological,stock2022embedding}, which investigates the possibility for one to recover the parameters (weights and biases) of a neural network. 

The complexity of neural networks in terms of their parameterization redundancy has received limited attention. Among the few relevant studies, \cite{grigsby2022local} and \cite{grigsby2022functional} are worth mentioning. In \cite{grigsby2022local}, the authors investigated local and global notions of topological complexity for fully-connected feedforward ReLU neural network functions. On the other hand, \cite{grigsby2022functional} defined the functional dimension of ReLU neural networks based on perturbations in parameter space. They explored functional redundancy and conditions under which the functional dimension reaches its theoretical maximum. However, it should be noted that these results on functional dimension do not directly translate into generalization error bounds for deep learning algorithms.

The complexity of a class of functions is closely related to generalization error, with larger complexities often leading to larger generalization error \citep{bartlett2017spectrally,mohri2018foundations}. Various complexity upper bounds have been studied for deep neural networks using different measurements, such as Rademacher complexity \citep{neyshabur2015norm,golowich2018size,li2018tighter}, VC-dimension and Pseudo dimension \citep{baum1988size,goldberg1993bounding,anthony1999neural,bartlett2019nearly}, and covering number \citep{anthony1999neural,neyshabur2017pac,bartlett2017spectrally,lin2019generalization}. These measurements characterize the complexity of the class of neural networks and are influenced by hyperparameters like network depth, width, number of weights and bias vectors, and corresponding norm bounds. While these bounds are not directly comparable in magnitude, they are closely related and can be converted to facilitate comparisons \citep{anthony1999neural,mohri2018foundations}.

\subsection{Our contributions}

We summarize our contributions as follows:	 	

\begin{itemize}	
	\item [1.] We make use of the permutation equivalence property to firstly obtain a tighter upper bound on the covering number of neural networks, which improves existing results by factorials of the network widths and provides unprecedented insights into the intricate relationship between network complexity and layer width.
	\item [2.] We improve existing covering number bounds in the sense that our results hold for neural networks with bias vectors and general activation functions. Since bias terms are indispensable for the approximation power of neural networks, our results are useful in both theory and practice. Additionally, we express our bound explicitly in terms of the network's width, depth, size, and the norm of the parameters. 
	\item [3.] We discuss the implications of our findings for understanding generalization and optimization. In particular, we found that overparameterized networks tend to be easier to train in the sense that increasing the width of neural networks leads to a vanishing volume of the effective parameter space.
\end{itemize}

The remainder of the paper is organized as follows. In section \ref{sec_funceq}, we introduce the concept of functional equivalence and investigate the permutation invariance property of general feedforward neural networks. In section \ref{sec_complexity}, we derive novel covering number bounds for shallow and deep neural networks by exploiting the permutation invariance property and compare our results with existing ones. 
In section \ref{sec_ext}, we discuss the extension to convolutional, residual and attention-based networks. In section \ref{sec_app}, we demonstrate the theoretical implication of permutation invariance on the optimization complexity in deep learning and discuss the implications of our results on generalization. Finally, we discuss the limitations of this study and future research directions in section \ref{sec_conclusion}. All technical proofs are included in the Appendix.

\begin{table*}[ht]
	\caption{A comparison of recent results on the complexity of feedforward neural networks.}
	\vskip 0.1in
	\centering
	\label{tab:compare}
	\begin{small}
		\begin{sc}
			\begin{tabular}{@{}C{2.9cm} C{4.5cm} C{1.35cm} C{1.45cm} C{1.95cm}@{}}
				\toprule
				Paper & Complexity & Explicit 
				& Bias Vectors &  Permutation Invariance \\ \midrule
				\cite{bartlett2017spectrally} & $B_x^2(\bar{\rho}\bar{s})^2\mathcal{U}\log(W)/\epsilon^2$ & \XSolidBrush & \XSolidBrush  & \XSolidBrush \\
				\cite{neyshabur2017pac} & $B_x^2(\bar{\rho}\bar{s})^2\mathcal{S}L^2\log(WL)/\epsilon^2$ & \XSolidBrush & \XSolidBrush & \XSolidBrush \\
				\cite{lin2019generalization}& $B_x(\bar{\rho}\bar{s})\mathcal{S}^2L/\epsilon$ & \XSolidBrush & \XSolidBrush  & \XSolidBrush \\
				\cite{bartlett2019nearly} & $L\mathcal{S}\log(\mathcal{S})\log(\bar{\rho}\bar{s}B_x/\epsilon)$  & \Checkmark & \Checkmark  & \XSolidBrush \\ 
				This paper & $L\mathcal{S}\log(\bar{\rho}\bar{s}B_x^{1/L}/((d_1!\cdots d_L!)^{1/\mathcal{S}}\epsilon)^{1/L})$ & \Checkmark & \Checkmark & \Checkmark \\ 
				\bottomrule
			\end{tabular}%
		\end{sc}
	\end{small}
	\begin{tablenotes}
		{\footnotesize
			\item Notations: $\mathcal{S}$ number of parameters; $\mathcal{U}$ number of hidden neurons; $L$ number of hidden layers;  $W$ maximum hidden layers width; $B_x$, L2 norm of input; $\bar{\rho}=\Pi_{j=1}^L\rho_j$, products of Lipschitz constants of activations; $\bar{s}=\Pi_{j=1}^Ls_j$, products of spectral norms of hidden layer weight matrices; $\epsilon$, radius for covering number.
		}
	\end{tablenotes}
\end{table*}

The remainder of the paper is organized as follows. In section \ref{sec_funceq}, we introduce the concept of functional equivalence and investigate the permutation invariance property of general feedforward neural networks. In section \ref{sec_complexity}, we derive novel covering number bounds for shallow and deep neural networks by exploiting the permutation invariance property and compare our results with existing ones. In section \ref{sec_ext}, we discuss the extension to convolutional, residual and attention-based networks. In section \ref{sec_app}, we demonstrate the theoretical implication of permutation invariance on the optimization complexity in deep learning and discuss the implications of our results on generalization. Finally, we discuss the limitations of this study and future research directions in section \ref{sec_conclusion}. All technical proofs are included in the Appendix.

	\section{Functionally equivalent Neural Networks}\label{sec_funceq}
		
		A feedforward neural network is a fully connected artificial neural network consisting of multiple layers of interconnected neurons. The network's architecture can be expressed as a composition of linear maps and activations. The functional form of an $L$-layer feedforward neural network is determined by its weight matrices, bias vectors, and activation functions:
		\begin{equation}\label{dnn}
			f(x;\theta) = \mathcal{A}_{L+1} \circ \sigma_L \circ \mathcal{A}_L \circ \cdots \circ \sigma_2 \circ \mathcal{A}_2 \circ \sigma_1 \circ \mathcal{A}_1(x).
		\end{equation}
		Here, $\mathcal{A}_l(x) = W^{(l)}x + b^{(l)}$ is the linear transformation for layer $l$, where $W^{(l)}$ and $b^{(l)}$ are the weight matrix and bias vector respectively. The activation function $\sigma_l$ is applied element-wise to the output of $\mathcal{A}_l$, and can be different across layers. The collection of weight matrices and bias vectors is denoted by $\theta = (W^{(1)}, b^{(1)}, \ldots, W^{(L+1)}, b^{(L+1)})$. The input $x$ is propagated through each layer of the network to produce the output $f(x;\theta)$.
		
		The parameterization of a neural network can be redundant, with different parameter sets producing identical function outputs. This redundancy arises from the non-identifiability of weight matrices or activation functions.

		\begin{definition}[Functionally-Equivalent Neural Networks]
			Two neural networks $f({x}; {\theta}_1)$ and $f({x}; {\theta}_2)$  are said to be functionally-equivalent on $\mathcal{X}$ if they produce the same input-output function for all possible inputs, i.e.,
			\begin{equation}
				f_1({x};{\theta}_1) = f_2({x}; {\theta}_2) \quad \forall {x} \in \mathcal{X},
			\end{equation}
			where $\mathcal{X}$ is the input space and ${\theta}_1$ and ${\theta}_2$ denote the sets of parameters of the two networks, respectively. 
		\end{definition}
		
		Neural networks with a fixed architecture can have functionally-equivalent versions through weight scaling, sign flips, and permutations. This can even occur across networks with different architectures. 
		In this paper, we focus on the complexity of a specific class of neural networks with fixed architecture but varying parameterizations. We provide examples of functionally-equivalent shallow neural networks to illustrate this concept. 
		
		\begin{example}[Scaling]\label{example_scaling}
			Consider two shallow neural networks parameterized by $\theta_1=(W^{(1)}_1,b_1^{(1)},W^{(2)}_1,b_1^{(2)})$ and $\theta_2=(W^{(1)}_2,b_2^{(1)},W^{(2)}_2,b_2^{(2)})$, defined as:
			\begin{align*}
				&f(x;\theta_1)=W^{(2)}_1\sigma(W^{(1)}_1x+b_1^{(1)})+b^{(2)}_1,\\ &f(x;\theta_2)=W^{(2)}_2\sigma(W^{(1)}_2x+b_2^{(1)})+b^{(2)}_2
			\end{align*}
			respectively, where $x \in \mathbb{R}^n$ is the input to the network and $\sigma$ satisfies
			$\sigma(\lambda x)=\lambda \sigma(x)$
			for all $x \in \mathbb{R}^n$ and $\lambda>0$. If there exists a scalar value $\alpha>0$ such that:
			\begin{align*}
				(W^{(1)}_2, b^{(1)}_2) = (\alpha W^{(1)}_1, \alpha b^{(1)}_1) \quad{\rm and} \quad W^{(2)}_2 = \frac{1}{\alpha} W^{(2)}_1,
			\end{align*}
			then $f(\cdot;\theta_1)$ and $f(\cdot;\theta_2)$ are functionally equivalent.
		\end{example}
		
		Scaling invariance property is applicable to ReLU, Leaky ReLU, and piecewise-linear activated neural networks. Specifically, for all ${x}\in \mathbb{R}^n$ and $\lambda\ge0$, we have $\sigma(\lambda x)=\lambda \sigma(x)$ for $\sigma$ being the ReLU or Leaky ReLU function. It is worth noting that the above example is presented for shallow neural networks, but the scaling invariance property can happen in deep networks across any two consecutive layers.

		\begin{example}[Sign Flipping]\label{example_sign}
			Consider two shallow neural networks $f(\cdot;\theta_1)$ and $f(\cdot;\theta_2)$ defined in Example \ref{example_scaling} with  $\sigma$ being an odd function, that is	$\sigma(-x)=-\sigma(x)$ for all $x \in \mathbb{R}^n$. 
			If 
			\begin{align*}
				(W^{(1)}_2, b^{(1)}_2) = (- W^{(1)}_1, - b^{(1)}_1) \quad{\rm and} \quad W^{(2)}_2 = - W^{(2)}_1,
			\end{align*}
			then $f(x;\theta_1)$ and $f(x;\theta_2)$ are functionally equivalent.
		\end{example}
		
		Sign flipping invariance property can happen for neural networks activated by Tanh, Sin and odd functions. It is worth noting that Sigmoid does not have a straightforward Sign flipping invariance. While Sigmoid is an odd function up to a constant 0.5, it can be Sign-flipping invariant up-to a constant and the constant can be mitigated by using a bias\citep{martinelli2023expand}. The sign flipping invariance property can also be generalized to deep neural networks across any two consecutive layers.
		
		\begin{example}[Permutation]\label{example_perm}
			Consider two shallow neural networks $f(\cdot;\theta_1)$ and $f(\cdot;\theta_2)$ defined in Example \ref{example_scaling} with  $\sigma$ being a general activation function. Let the dimension of the hidden layer of $f(x;\theta_1)$ and $f(x;\theta_2)$ be denoted by $m$. If there exists an $m \times m$ permutation matrix $P$ such that
			\begin{align*}
				(P W^{(1)}_2, P b^{(1)}_2) = (W^{(1)}_1, b^{(1)}_1) \quad{\rm and} \quad W^{(2)}_2 P = W^{(2)}_1,
			\end{align*}
			then $f(x;\theta_1)$ and $f(x;\theta_2)$ are functionally equivalent.
		\end{example}
		
		The feedforward neural networks are built of linear transformations and activations, and it is intuitive that simply re-indexing neurons in a hidden layer and the corresponding rows of the weights matrix and bias vector will lead to a functionally equivalent network. The permutation invariance is the most basic type of equivalence for neural networks since it does not rely on any specific properties of activation functions, while scaling and sign flipping invariance are activation-dependent properties. A comparison on functional equivalence properties on neural network with commonly-used activation functions is presented in Table \ref{tab:1}.

		\begin{table*}[ht]
			\centering
			\caption{Functional equivalence property for networks with different activation functions.}
			\label{tab:1}
			\vskip 0.1in
			\begin{small}
				\begin{sc}
					\begin{tabular}{@{}C{2cm} C{5cm} C{2cm} C{1.8cm} C{2cm} @{}}
						\toprule
						Activation & Formula                       & Sign flipping             & Scaling                   & Permutation               \\ \midrule
						Sigmoid    & $[1+\exp(-x)]^{-1}$      & \XSolidBrush & \XSolidBrush & \Checkmark \\
						Tanh                 & $[1-\exp(-2x)]/[1+\exp(-2x)]$ & \Checkmark    & \XSolidBrush & \Checkmark \\
						ReLU                 & $\max\{0,x\}$                 & \XSolidBrush & \Checkmark & \Checkmark \\
						Leaky ReLU & $\max\{ax,x\}$ for $a>0$ & \XSolidBrush & \Checkmark   & \Checkmark \\ \bottomrule
					\end{tabular}%
				\end{sc}
			\end{small}
		\end{table*}
		
		Next, we derive sufficient conditions fo feed-forward neural networks (FNNs) to be permutation-equivalent.
		
		\begin{proposition}[Permutation equivalence for deep FNNs]\label{thm_perm}
			Consider two neural networks $f(x;\theta_1)$ and $f(x;\theta_2)$ with the same activations $\sigma_1,\ldots,\sigma_L$ and architecture  
			{\footnotesize
				$$f(x;\theta) = W^{(L+1)} \sigma_L(\cdots\sigma_1(W^{(1)}_1 x+b^{(1)}_1)\cdots)+ b^{(L)}_1) +b^{(L+1)}$$
			}
			but parameterized by different parameters
			\begin{align*}
				&\theta_j = (W^{(1)}_j, b^{(1)}_j, \ldots, W^{(L+1)}_j, b^{(L+1)}_j),\quad j=1,2
			\end{align*}
			respectively, where $x \in \mathbb{R}^n$ is the input to the network. Let 
			$P^\top$ denote the transpose of matrix $P$. If there exists permutation matrices $P_1,\ldots,P_{L}$ such that
			\begin{align*}
				W^{(1)}_1 = P_{1} W^{(1)}_2, \qquad\qquad\hspace*{0.15cm} & b^{(1)}_1 = P_{1} b^{(1)}_2,\\
				W^{(l)}_1 = P_{l} W^{(l)}_2 P_{l-1}^\top, \qquad\quad & b^{(l)}_1 = P_{l} b^{(1)}_2, \quad l=2,\ldots,L\\
				W^{(L+1)}_1 = W^{(L+1)}_2 P_{L}^\top, \qquad & b^{(L)}_1 = b^{(L)}_2,
			\end{align*}
			then $f(x;\theta_1)$ and $f(x;\theta_2)$ are functionally equivalent.
		\end{proposition}
		
		Proposition \ref{thm_perm} describes the relationship between the parameters of two permutation-equivalent deep feedforward neural networks. This relationship can be used to create functionally equivalent networks given fixed architectures. It's important to note that although permutation invariance is sufficient for functional equivalence of feedforward neural networks, it's not always necessary. \cite{petzka2020notes} gave a complete characterization for fully-connected networks with two layers. While for general networks, certain restrictions on the architecture and activation function are required to fully characterize  \citep{sussmann1992uniqueness,kuurkova1994functionally,bui2020functional} and recover the parameters of a network \citep{martinelli2023expand}. This study focuses only on utilizing permutation invariance to investigate neural network complexity.

		In this section, we analyze the complexity of a class of feedforward neural networks by examining the redundancy that arises from permutation invariance. Specifically, we study the covering number of real-valued, deep feedforward neural networks that share the same architecture but have different parameterization.
		
		Let the vector $(d_0, d_1, \ldots, d_L)$ represent the dimensions of the layers of the neural network $f(x;\theta)$ defined in (\ref{dnn}), where $d_{L+1}=1$ as the output is real-valued. Note that the bias vectors in hidden layers contain $\mathcal{U}:=\sum_{i=1}^{L} d_i$ entries, and the weight matrices together with bias vectors contain $\mathcal{S}:=\sum_{i=0}^{L} d_i\times d_{i+1}+d_{i+1}$ entries in total. We define the parameter space of $\theta$ as $\Theta=[-B,B]^\mathcal{S}$ for some $B\ge 1$, which is closed for permutation operations and ensures the absolute value of the weight matrix and bias vector entries are bounded by $B$. We set $\Theta=[-B,B]^\mathcal{S}$ for some $B\ge 1$. The setting is in line with complexity studies as in \citep{neyshabur2015norm,bartlett2017spectrally,golowich2018size} with norm controls. The setting of bounded parameter space can also correspond to the observed implicit regularization phenomena in SGD-based optimization algorithms \citep{neyshabur2014search,gunasekar2018characterizing,gunasekar2018implicit}, which can lead to minimum-norm solutions (e.g., for least square problems). We do not specify any activation functions $\sigma_1,\ldots,\sigma_L$ since we consider general deep feedforward neural networks. Finally, the class of feedforward neural networks we consider is denoted as

			\begin{align}\label{dnns}
				&\mathcal{F}(L,d_0,d_1,\ldots,d_L,B)=\{ f(\cdot;\theta): \mathbb{R}^{d_0}\to \mathbb{R} {\rm\ is\ defined\ in\ (\ref{dnn})}:\theta\in[-B,B]^\mathcal{S}\}.
			\end{align}

		\subsection{Shallow Feed-Forward Neural Networks}
		It is well-known that a shallow neural network with a single hidden layer has universal approximation properties \citep{cybenko1989approximation,hornik1991approximation} and is sufficient for many learning tasks \citep{hanin2019deep,hertrich2021towards}. We begin our investigation by considering shallow neural networks of the form
			\begin{align}\label{shallow}
				&\mathcal{F}(1,d_0,d_1,B)=\{ f(x;\theta)=W^{(2)}\sigma_1(W^{(1)}x+b^{(1)})+b^{(2)}: \theta\in[-B,B]^{\mathcal{S}}\}
		\end{align}
		where the total number of parameters is given by $\mathcal{S}=(d_0+2)\times d_1+1$. By Theorem \ref{thm_perm}, for any $\theta=(W^{(1)},b^{(1)},W^{(2)},b^{(2)})\in\Theta$ and any permutation matrix $P$, $\tilde{\theta}=(PW^{(1)},Pb^{(1)},W^{(2)}P^\top,b^{(2)})\in\Theta$ will produce the same input-out function. Actually, permutation invariance leads to equivalence classes of parameters that yield the same realization, and we can obtain a set of representatives from equivalence classes. A canonical choice is
		$$\Theta_0:=\{\theta\in[-B,B]^{\mathcal{S}}: b^{(1)}_1\ge b^{(1)}_2\ge\cdots\ge b^{(1)}_{d_1} \},$$
		where the set of representatives is by restricting the bias vector $b^{(1)}=(b^{(1)}_1,\ldots,b^{(1)}_{d_1})^\top$ to have descending components. Alternatively, we can sort the first component of the rows of $W^{(1)}$ to obtain a set of representatives.	
		It is worth mentioning that $\Theta_0$ may not be the minimal set of representatives since there may be other symmetries within $\Theta_0$. We did not further utilize these additional symmetries to reduce  $\Theta_0$ to be smaller since other symmetries can depend on the specific properties of the activation functions in the neural networks. In this work, we specifically employ permutation invariance (which holds for networks with any activations) to obtain a representative set $\Theta_0$. 
		
		The set of representatives $\Theta_0$ has two important properties
		\begin{itemize}
			\item Neural networks $\{f(\cdot;\theta):\theta\in\Theta_0\}$ parameterized by $\Theta_0$ contains all the functions in $\{ f(\cdot;\theta):\theta\in\Theta\}$, i.e.,
			$$\{f(\cdot;\theta):\theta\in\Theta_0\}=\{f(\cdot;\theta):\theta\in\Theta\}.$$
			\item The volume (in terms of Lebesgue measure) of the set of representatives $\Theta_0$ is $(1/d_1!)$ times smaller that that of the parameter space $\Theta$, i.e.,
			$${\rm Volume}(\Theta)=d_1!\times{\rm Volume}(\Theta_0).$$
		\end{itemize}
		The first property holds naturally since any parameter $\theta\in\Theta$ has a permuted version in $\Theta_0$. Regarding the second property, we note that both $\Theta$ and $\Theta_0$ belong to the Euclidean space $\mathbb{R}^{\mathcal{S}}$ (where $\mathcal{S}$ denotes the number of parameters), then their volumes in terms of Lebesgue measure can be calculated. For any parameter $\theta$ with distinct components in the bias vector $b^{(1)}=(b^{(1)}_1,\ldots,b^{(1)}_{d_1})^\top$, permutation of the bias vector and corresponding weights can lead to $(d_1!)$ distinct equivalent parameters $\theta_1,\ldots,\theta_{d_1!}$. Then this implies a $(d_1!)$ times smaller volume of representative set $\Theta_0$ compared to that of $\Theta$.
		
		Subsequently, the two properties suggest that $\Theta_0$ can be a representative parameterization of neural networks when they are viewed only as input-output functions. Based on these observations, we derive improved complexities of the class of neural networks in terms of its covering number.
		\begin{definition}[Covering Number]
			Let $\mathcal{F}={f:\mathcal{X}\to\mathbb{R}}$ be a class of functions. We define the supremum norm of $f\in\mathcal{F}$ as $\Vert f\Vert_\infty:=\sup_{x\in\mathcal{X}}\vert f(x)\vert$. For a given $\epsilon>0$, we define the covering number of $\mathcal{F}$ with radius $\epsilon$ under the norm $\Vert \cdot\Vert_\infty$ as the least cardinality of a subset $\mathcal{G}\subseteq\mathcal{F}$ satisfying
			$$\sup_{f\in\mathcal{F}}\min_{g\in\mathcal{G}}\Vert f-g\Vert_\infty\le \epsilon.$$
			Denoted by $\mathcal{N}(\mathcal{F},\epsilon,\Vert\cdot\Vert_\infty)$, the covering number measures the minimum number of functions in $\mathcal{F}$ needed to cover the set of functions within a distance of $\epsilon$ under the supremum norm.
		\end{definition}
		The covering number $\mathcal{N}(\mathcal{F},\epsilon,\Vert\cdot\Vert_\infty)$ provides a quantitative measure of the complexity of the class of functions $\mathcal{F}$ under the supremum norm, with smaller values indicating simpler classes. Covering numbers, along with Rademacher complexity, VC dimension, and Pseudo dimension, are essential complexity measures in the analysis of learning theories and in estimating generalization errors. Although these measures are different, they are correlated with each other, and we introduce the detailed correlations in Appendix.
		
		\begin{remark}
			We define the covering number of a class of functions in the uniform sense. This is an extension of the canonical definition of covering numbers, which was originally developed for subsets in Euclidean space. While most existing studies of covering numbers for function spaces consider the image of the functions on a finite sample \citep{anthony1999neural,bartlett2017spectrally}, our definition is formulated directly in terms of the function space itself, without requiring a finite sample or any other auxiliary construction. 
		\end{remark}
		
		\begin{theorem}[Covering number of shallow neural networks]\label{cv_shallow}
			Consider the class of single hidden layer neural networks $\mathcal{F}:=\mathcal{F}(1,d_0,d_1,B)$ defined in (\ref{shallow}) parameterized by $\theta\in\Theta=[-B,B]^\mathcal{S}$. Suppose the radius of the domain $\mathcal{X}$ of $f\in\mathcal{F}$ is bounded by some $B_x>0$, and the activation $\sigma_1$ is continuous. Then for any $\epsilon>0$, the covering number
			\begin{align}\label{cv_shallow1}
				\mathcal{N}(\mathcal{F},\epsilon,\Vert\cdot\Vert_\infty)\le (16B^2(B_x+1)\sqrt{d_0}d_1/\epsilon)^\mathcal{S} \times\rho^{\mathcal{S}_h}/d_1!,
			\end{align}
			where $\rho$ denotes the Lipschitz constant of $\sigma_1$ on the range of the hidden layer (i.e., $[-\sqrt{d_0}B(B_x)+1),\sqrt{d_0}B(B_x+1)]$), and $\mathcal{S}_h=d_0d_1+d_1$ is the total number of parameters in the linear transformation from input to the hidden layer, and $\mathcal{S}=d_0\times d_1+2d_1+1$ is the total number of parameters.
		\end{theorem}
		Our upper bound on the covering number firstly takes advantage of permutation invariance, resulting in a reduced complexity (by a factorial term $d_1!$ in the denominator) compared to existing studies \citep{neyshabur2015norm,bartlett2017spectrally,neyshabur2017pac,neyshabur2017implicit,lin2019generalization}. The factorial reduction $d_!!$ can be significant. For instance, for a shallow ReLU network with a hidden dimension of $d_1=128$, the factorial $128!\approx10^{215}$, which is far larger than $10^{82}$, the upper estimate on the number of atoms in the known universe. This reduction can be substantial and can enhance theoretical analysis and results that rely on covering numbers. In addition, it's worth noting that bounds in Theorem \ref{cv_shallow} holds true for networks with bias vectors. This is also an improvement over existing studies that didn't consider bias vectors in neural networks, since bias terms are crucial for the approximation power of neural networks \citep{yarotsky2017error,lu2021deep,shen2022approximation}. 
		
		Stirling's formula can be used to approximate the factorial term as $\sqrt{2\pi d_1}(d_1/e)^{d_1}\exp(1/(12d_1+1))<d_1!<\sqrt{2\pi d_1}(d_1/e)^{d_1}\exp(1/12d_1)$ when $d_1\ge1$. This reduces the covering number approximately by a factor of $(d_1/e)^{d_1}$ and the bound in (\ref{cv_shallow1}) is $(C_{B,B_x,d_0,\rho}/\epsilon)^\mathcal{S}\times d_1^{\mathcal{S}-\mathcal{U}}$, where $\mathcal{U}=d_1$ denotes the number of hidden neurons and  $C_{B,B_x,d_0,\rho}>0$ is a constant depending only on $B,B_x,d_0$ and $\rho$.  Notably,  $\mathcal{S}-\mathcal{U}$ is the number of weights (excluding bias), and the reduced bound is basically that for a single neural network without bias vectors and considering permutation invariance.
		Lastly, we note that increasing the number of neurons in a shallow neural network enlarges its approximation power \citep{lu2017expressive,ongie2019function}, but at a smaller increase in complexity according to our results.
		
		\begin{remark} \label{remark4}
			Theorem \ref{cv_shallow} applies to any activation function that is locally Lipschitz on bounded sets (range of the hidden layer), which does not require any specific choice such as Hinge or ReLU in \cite{neyshabur2015norm,neyshabur2017pac}, and does not require universal Lipschitz and $\sigma(0)=0$ as in \cite{bartlett2017spectrally,lin2019generalization}. In the case of the ReLU or Leaky ReLU activation, our bound simplifies to $\rho=1$ without any condition, leading to the disappearance of the term $\rho^{\mathcal{S}_h}$ in our bound.
		\end{remark}

		\subsection{Deep Feed-Forward Neural Networks}\label{sec_cvdeep}
		For deep neural networks, we can also analyze its effective parameter space based on permutation invariance properties. 
		By Theorem \ref{thm_perm}, a set of representatives $\Theta_0=\Theta_0^{(1)}\times \Theta_0^{(2)}\times\cdots\times\Theta_0^{(L)}\times\Theta_0^{(L+1)}$ can be constructed where

			\begin{align*}
				&\Theta_0^{(l)} = \Big\{(W^{(l)},b^{(l)})\in[-B,B]^{\mathcal{S}_l}: b^{(l)}_1\ge b^{(l)}_2\ge\cdots\ge b^{(l)}_{d_l} \}\\
				&\qquad\qquad{\rm \ for\ }l=1,\ldots,L, \Theta_0^{(L+1)} =\{(W^{(L+1)},b^{(L+1)})\in[-B,B]^{\mathcal{S}_{L+1}}\Big\}.
			\end{align*}
			Then we can obtain an upper bound of the covering number of deep feedforward neural networks.	
		
		\begin{theorem}[Covering number of deep neural networks]\label{cv_deep}
			Consider the class of deep neural networks $\mathcal{F}:=\mathcal{F}(1,d_0,d_1,\ldots,d_L,B)$ defined in (\ref{dnns}) parameterized by $\theta\in\Theta=[-B,B]^\mathcal{S}$. Suppose the radius of the domain $\mathcal{X}$ of $f\in\mathcal{F}$ is bounded by $B_x$ for some $B_x>0$, and the activations $\sigma_1,\ldots,\sigma_L$ are locally Lipschitz. Then for any $\epsilon>0$, the covering number $\mathcal{N}(\mathcal{F},\epsilon,\Vert\cdot\Vert_\infty)$ is bounded by
				\begin{align*}\label{cv_deep1}
					\frac{\Big(4(L+1)(B_x+1)(2B)^{L+2}(\Pi_{j=1}^L \rho_j) ({\Pi_{j=0}^Ld_{j}})\cdot\epsilon^{-1}\Big)^{\mathcal{S}}}{d_1!\times d_2!\times\cdots\times d_L!},
			\end{align*}
			where $\mathcal{S}=\sum_{i=0}^{L}d_id_{i+1}+d_{i+1}$ and $\rho_i$ denotes the Lipschitz constant of $\sigma_i$ on the range of $(i-1)$-th hidden layer, especially the range of $(i-1)$-th hidden layer is bounded by $[-B^{(i)},B^{(i)}]$ with $B^{(i)}\le (2B)^i\Pi_{j=1}^{i-1}\rho_j{d_j}$ for $i=1,\ldots,L$.
		\end{theorem}
		
		Theorem \ref{cv_deep} provides a novel upper bound for the covering number of deep neural networks based on permutation invariance, which reduces the complexity compared to previous results \citep{neyshabur2015norm,neyshabur2017pac,bartlett2017spectrally,lin2019generalization} by approximately a factor of $(d_1!d_2!\cdots d_L!)$. According to Theorem \ref{cv_deep}, increasing the depth of a neural network increases its complexity. However, it is interesting to note that the increased hidden layer $l$ will have a $(d_l!)$ discount on the complexity. If the hidden layers have equal width ($d=d_1=\cdots=d_L$), the bound reduces to $(C_{B,B_x,d_0,\rho}/\epsilon)^\mathcal{S}\times d^{\mathcal{S}-\mathcal{U}}$, where $\mathcal{U}=Ld$ denotes the number of hidden neurons and $C_{B,B_x,d_0,\rho}>0$ is a constant depending only on $B,B_x,d_0$ and $\rho_i,i=1,\ldots,L$. As with shallow neural networks, the improved rate $(\mathcal{S}-\mathcal{U})$ denotes the number of weights but excluding biases, which assures the approximation power of neural networks, but grows its complexity in a rate free of number of bias.
		
		\begin{remark}
			As discussed in Remark \ref{remark4}, our results take into account the permutation invariance and have looser requirements on activation functions compared to existing results \citep{neyshabur2015norm,neyshabur2017pac,bartlett2017spectrally,lin2019generalization}. In addition, our upper bound is explicitly expressed in parameters which are known and can be specified in practice, e.g., network depth $L$, width $(d_0,d_1,\ldots,d_L)$, size $\mathcal{S}$ and uniform bound $B$ for weights and biases. While most existing bounds in \citep{neyshabur2015norm,neyshabur2017pac,bartlett2017spectrally,lin2019generalization} are in terms of the spectral norm of weight matrices and some measurement on externally introduced reference matrices, which are usually unknown in practice.
		\end{remark}

		\subsection{Comparing to existing results}
		
		The complexity upper bounds in terms of covering number on the class of deep neural networks have been studied in \cite{anthony1999neural,neyshabur2017pac,bartlett2017spectrally,lin2019generalization}.
		These results are proved by using similar approaches of mathematical induction (e.g. Lemma A.7 in \cite{bartlett2017spectrally}, Lemma 2 in \cite{neyshabur2017pac} and Lemma 14 of \cite{lin2019generalization}). Compared with these results, we improve upon their results in three ways.
		\begin{itemize}
			\item  First, we consider the generally defined neural networks where the bias vector is allowed to appear.  Bias terms are indispensable for the approximation power of neural networks \citep{yarotsky2017error,lu2021deep,shen2022approximation}. Neural networks without bias vectors may be disqualified in theory and practice.
			\item  Second, we explicitly express our bound in terms of the width, depth, and size of the neural network, as well as the infinity norm of the parameters.
			\item Third, we utilize the permutation equivalence property to obtain a tighter upper bound on the covering number. Notably, this bound improved existing results by a factorial of the network widths, providing unprecedented insights into the intricate relationship between network complexity and layer width.
		\end{itemize}

		Various complexity upper bounds for deep neural networks in other measurements have also been studied, including Rademacher complexity \citep{neyshabur2015norm,golowich2018size,li2018tighter}, VC-dimension, and Pseudo dimension \citep{baum1988size,goldberg1993bounding,anthony1999neural,bartlett2019nearly}. Our bounds in terms of covering number are not directly comparable with these measurements. To make a comparison with these other measurements, we convert the upper bounds into metric entropy, which is equivalent to the logarithm of the covering number $\log(\mathcal{N}(\mathcal{F},\epsilon,\Vert\cdot\Vert_\infty))$. Specially, let $\bar{\rho}=\Pi_{j=1}^L\rho_j$ products of Lipschitz constants of activation functions and $\bar{s}=\Pi_{j=1}^Ls_j$ products of spectral norms of hidden layer weight matrices \cite{bartlett2017spectrally} derived an spectral-norm based bound of metric entropy $B_x^2(\bar{\rho}\bar{s})^2\mathcal{U}\log(W)/\epsilon^2$, following which \cite{neyshabur2017pac} obtained $B_x^2(\bar{\rho}\bar{s})^2\mathcal{S}L^2\log(WL)/\epsilon^2$ and \cite{lin2019generalization} obtained $B_x(\bar{\rho}\bar{s})\mathcal{S}^2L/\epsilon$. Based on Theorem 12.2 in \cite{anthony1999neural}, the Pseudo dimension bound in \cite{bartlett2017spectrally} leads to $L\mathcal{S}\log(\mathcal{S})\log(\bar{\rho}\bar{s}B_x/\epsilon)$. Lastly, our results in Theorem \ref{cv_deep} can be presented as $L\mathcal{S}\log(\bar{\rho}\bar{s}B_x^{1/L}/((d_1!\cdots d_L!)^{\mathcal{S}}\epsilon)^{1/L})$ by letting $s_i:=B\sqrt{d_id_{i-1}}$ given in our setting. It is important to note that the quantity $B\sqrt{d_id_{i-1}}$ can provide an upper bound for the spectral norm $s_i$, but the reverse is not necessarily true. We cannot ensure Theorem \ref{cv_deep} to still be true by directly substituting $B\sqrt{d_id_{i-1}}$ with $s_i$ in the theorem. In other words, if the covering number bounds are derived in terms of the spectral norm, the reduction factor $d_1!\cdots d_L!$ may not be obtained, and additional parameters like matrix norm bounds would appear in the upper bound\citep{bartlett2017spectrally}. We would also mention that even our results improved over the existing ones, all these bounds scale with number of parameters $\mathcal{S}$ and can still result vacuous bounds in error analysis with extremely over-parametrized settings. We present a detailed comparison of the results in Table \ref{tab:compare}.

	\section{Extension to other neural networks}\label{sec_ext}
	Functional equivalence can manifest ubiquitously across various types of neural networks, with a specific emphasis on the presence of permutation equivalence within neural networks featuring linear transformation layers. This section delves into the exploration of feasible extensions aimed at harnessing the power of functional equivalence within convolutional neural networks, residual networks, and attention-based networks.
	\subsection{Convolutional neural networks}
	Convolutional neural networks (CNNs) are featured by the utilization of convolution and pooling layers. In a convolution layer, the input is convolved with parameter-carrying filters, resembling a linear layer with a sparse weight matrix. A pooling layer is commonly employed for downsampling and summarizing input feature information. It partitions the input into non-overlapping regions and applies a pooling operation to each region. The most prevalent types of pooling operations include max/min/avg pooling, which retain the maximum, minimum, or average value within each region.
	
	As demonstrated in Example \ref{example_scaling}-\ref{example_perm}, scaling, sign flip, and permutation equivalence directly apply to convolution layers (linear layer with sparse weight matrix). We also extend the permutation equivalence within pooling regions as follows.
	
	\begin{example}[Permutation within Pooling Regions]\label{example_pool}
		Consider two shallow CNNs defined by $f(x;\theta_1)={Pool}(W_1x+b_1)$ and $f(x;\theta_2)={Pool}(W_2x+b_2)$ respectively where ``$Pool$" is a pooling operator. Let $\mathcal{I}_1,\ldots,\mathcal{I}_K$ be the non-overlapping index sets (correspond to the pooling operator) of rows of  $W_1x+b_1$ and $W_2x+b_2$. Then $f(\cdot;\theta_1)$ and $f(\cdot;\theta_2)$ are functional equivalent if there exists a permutation matrix $P$ such that $\forall k\in\{1,\ldots,K\}$
		$$(P W_2)_{\mathcal{I}_k} \cong (W_1)_{\mathcal{I}_k}\  {\rm and}\  (P b_2)_{\mathcal{I}_k} \cong (b_1)_{\mathcal{I}_k},$$
		where $A_{\mathcal{I}_k}$ denotes the $\mathcal{I}_k$ rows of $A$ and $A\cong B$ denotes that $A$ equals to $B$ up to row permutations. 
	\end{example}
	
	Permutation equivalence within non-overlapping regions in CNNs preserves max/min/avg values, eliminating the need for cancel-off operations in subsequent layers. This further allows for the derivation of complexity bounds of CNNs by identifying their effective parameter space, similar to Theorem \ref{cv_deep}. In addition, the results on the connections between CNNs and feed-forward networks can be naturally utilized to improve the covering number bounds and capacity estimates \citep{fang2020theory,mao2021theory},  e.g., \citep{zhou2020universality,zhou2020theory} showed that fully connected deep ReLU networks can be realized by deep CNNs with the same order of network parameter numbers,.
	
	\subsection{Residual Networks}
	
	Residual Networks are a type of deep CNN architecture that has a significant impact on computer visions \citep{he2016deep}. The key feature of ResNet is the use of skip connections that enable networks to learn residual mappings and bypass layers, leading to very deep but trainable networks.  Mathematically, a residual layer $f(x;\theta) = x + \text{F}(x;\theta)$ outputs the summation of the input $x$ and its transformation $\text{F}(x;\theta)$. The $F(\cdot;\theta)$ can be any transformation that maps $x$ to the space of itself, and it defines the residual layer. Then the equivalence of $F$ implies that of the residual layer.
	\begin{example}[Equivalence of Residual Layer]\label{pattern5}
		Consider two residual layers $f(x;\theta_1)=x+F(x;\theta_1)$ and $f(x;\theta_2)=x+F(x;\theta_2)$. Then $f(\cdot;\theta_1)$ and $f(\cdot;\theta_2)$ are functionally equivalent if and only if $F(\cdot;\theta_1)$ and $F(\cdot;\theta_2)$ are functionally equivalent.
	\end{example}
	
	
	\subsection{Attention-based Networks}
	
	Attention-based models, well known for BERT, GPT and many others, have been successful in natural language processing and computer vision tasks \citep{vaswani2017attention,devlin2018bert,radford2018improving}. It utilizes an attention mechanism to focus on relevant parts of the input data. Here we focus on the self-attention module due to its effectiveness. Let $X_{n\times d}$ denote the input of a $n$-sequence of $d$-dimensional embeddings, and let $W^Q_{d\times d_q},W^K_{d\times d_k}$ and $W^V_{d\times d_v}$ be the weight matrices where $d_q=d_k$. Then the self-attention map outputs  $$Softmax\left(\frac{XW^Q(W^K)'X'}{\sqrt{d_k}}\right)XW^V$$
	where the $Softmax(\cdot)$ is applied to each row of its input and $A^\prime$ denotes the transpose of a matrix $A$.
	\begin{example}[Permutation within Attention map]\label{pattern6}
		Consider two attention maps $f(x;\theta_1)$ and $f(x;\theta_2)$ with $f(x;\theta_i)=Softmax({XW_i^Q(W_i^K)'X'}/{\sqrt{d_k}})XW_i^V$ for $i=1,2$. Then $f(\cdot;\theta_1)$ and $f(\cdot;\theta_2)$ are functionally equivalent if there exists $d_k\times d_k$ permutation matrix $P$ such that
		$$W_2^QP=W^Q_1\quad {\rm and}\quad W_2^KP=W^K_1.$$
	\end{example}
	For the attention module, there is no activation function between the key and query matrices. The relevant symmetry can be considered for any equivalent linear maps $W^Q(W^K)'$.  In addition, the output of $Softmax$ operator is invariant to the row shift of its input, which also leaves the possibility to further reduce the complexity of attention modules to understand the overparameterization of large language models.

		\section{Implications to generalization and optimization} \label{sec_app}
		In this section, we introduce the relevance and highlight the usefulness of our study to both generalization and optimization via empirical risk minimization (ERM) framework. 
		
		The goal of ERM is to find the target function $f_0$, which represents the true relationship between the inputs and outputs, and is typically defined as the minimizer (can be unbounded) of some risk $\mathcal{R}(\cdot)$, i.e., 
		$f_0:=\arg\min_{f}\mathcal{R}(f).$ However, since the target function is unknown, we can only approximate it using a predefined hypothesis space $\mathcal{F}$, such as the class of neural networks parameterized by $\theta$ in deep learning, i.e., $\mathcal{F}(\Theta)=\{f_\theta(\cdot)=f(\cdot;\theta):\theta\in\Theta\}$. Then the ``best in class" estimator is defined by 
		$f_{\theta^*}=\arg\min_{f\in\mathcal{F}_\Theta}\mathcal{R}(f).$
		It's worth noting that the risk function $\mathcal{R}$ is defined with respect to the distribution of the data, which is unknown in practice. Instead, only a sample with size $n$ is available, and the empirical risk $\mathcal{R}_n$ can be defined and minimized to obtain an empirical risk minimizer (ERM), i.e.,
		$f_{\theta_n}\in\arg\min_{f\in\mathcal{F}_\Theta}\mathcal{R}_n(f).$
		Finally, optimization algorithms such as SGD and ADAM, lead us to the estimator obtained in practice, i.e.,  $f_{\hat{\theta}_{n}}$.
		The generalization error of $f_{\hat{\theta}_n}$ can be defined and decomposed as \citep{mohri2018foundations}:
			\begin{align*}
				&\underbrace{\mathcal{R}(f_{\hat{\theta}_{n,opt}})-\mathcal{R}(f_0)}_{\rm generalization\  error}= \underbrace{\mathcal{R}(f_{\hat{\theta}_{n}})-\mathcal{R}(f_{\theta_n})}_{\rm optimization\  error}+\underbrace{\mathcal{R}(f_{\theta_n})-\mathcal{R}(f_{\theta^*})}_{\rm estimation\  error}+\underbrace{\mathcal{R}(f_{\theta^*})-\mathcal{R}(f_{0})}_{\rm approximation\  error}.
			\end{align*}

		The estimation error is closely related to the complexity of the function class $\mathcal{F}(\Theta)$ and the sample size $n$. pecifically, for a wide range of problems such as regression and classification, the estimation error is $\mathcal{O}((\log\{\mathcal{N}(\mathcal{F}(\Theta),1/n,\Vert\cdot\Vert_\infty)\}/n)^{k})$ for $k=1/2$ or 1 \citep{bartlett2019nearly,kohler2021rate,shen2022approximation,jiao2023deep}. Our results improve the estimation error by subtracting at least $\log(d_1!\cdots d_L!)$ from the numerator $(\cdot/n)^{k}$ compared to existing results.
		
		The approximation error depends on the expressive power of networks $\mathcal{F}(\Theta)$ and the features of the target $f_0$, such as its input dimension $d$ and smoothness $\beta$. Typical results for the bounds of the approximation error are $\mathcal{O}((L\mathcal{W})^{-\beta/d})$ \citep{yarotsky2017error,yarotsky2018optimal,petersen2018optimal,lu2021deep} where $L$ and $\mathcal{W}$ denote the depth and width of the neural network. However, it is unclear how our reduced covering number bounds will improve the approximation error based on current theories.

		Regarding the optimization error,  due to the high non-convexity and complexity of deep learning problems, quantitative analysis based on current theories is limited. Even proving convergence (to stationary points) of existing methods is a difficult task \citep{sun2020optimization}. However, we found that the symmetric structure of the parameter space can facilitate optimization.  To be specific, our Theorem \ref{opt} in the following indicates that considering the symmetry structure of the deep network parameter space can make the probability of achieving zero (or some level of) optimization error $(d_1!\cdots d_L!)$ times larger.

		For a deep neural network in (\ref{dnns}), we say two rows in the parameters $\theta^{(l)}:=(W^{(l)};b^{(l)})$ in the $l$th hidden layer are identical if the two rows of the concatenated matrix $(W^{(l)};b^{(l)})$ are identical. Here we concatenate the weight matrix $W^{(l)}$ and bias vector $b^{(l)}$ by $(W^{(l)};b^{(l)})$ due to the one-one correspondence of the rows in $W^{(l)}$ and $b^{(l)}$. Specifically, if the $l$th layer of the network is activated by $\sigma$, then the $i$th row of the output vector $\sigma(W^{(l)}x+b^{(l)})$ is given by $\sigma(W^{(l)}_ix+b^{(l)}_i)$ where $W^{(l)}_i$ and $b^{(l)}_i$ denote the $i$th row of $W^{(l)}$ and $b^{(l)}$ respectively.
		We let $d_l^*$ denote the number of distinct permutations of rows in $\theta^{(l)}$, and let $(d_1^*,\ldots,d_L^*)$ collect the number of distinct permutations in the hidden layers of the network parameterization $\theta=(\theta^{(1)},\ldots,\theta^{(L)})$. We let $\Delta_{\rm min}(\theta)$ and $\Delta_{\rm max}(\theta)$ denote the minimum and maximum of the $L_\infty$ norm of distinct rows in $\theta^{(l)}$ over $l\in{1,\ldots,L}$ (see Definition (\ref{delta_min}) and (\ref{delta_max}) in Appendix for details). Then we have the following result.
		
		\begin{theorem}\label{opt}
			Suppose we have an ERM $f_{\theta_{n}}(\cdot)=f(\cdot;\theta_{n})$ with parameter $\theta_{n}$ having $(d_1^*,\ldots,d_L^*)$ distinct permutations and $\Delta_{\rm min}(\theta_{n})=\delta$. For any optimization algorithm $\mathcal{A}$, if it guarantees producing a convergent solution of $\theta_{n}$ when its initialization $\theta^{(0)}_{n}$ satisfies $\Delta_{\rm max}(\theta^{(0)}_{n}-\theta_{n})\le\delta/2$, then any initialization scheme that uses identical random distributions for the entries of weights and biases within a layer will produce a convergent solution with probability at least $d_1^*\times\cdots \times d_L^*\times \mathbb{P}(\Delta_{\rm max}(\theta^{(0)}-\theta_n)\le \delta/2)$. Here, $\theta^{(0)}$ denotes the random initialization, and $\mathbb{P}(\cdot)$ is with respect to the randomness from initialization.
		\end{theorem}
		Theorem \ref{opt} can be understood straightforwardly. By Theorem \ref{thm_perm}, if $f_{\theta_{n}}$ parameterized by $\theta_{n}$ is a solution, then $f_{\tilde{\theta}_{n}}$ parameterized by $\tilde{\theta}_{n}$ (the permuted $\theta_n$) is also a solution. The conditions in Theorem \ref{thm_perm} 
		ensure that the convergent regions for permutation-implemented solutions are disjoint. Thus, the probability fro convergence can be multiplied by the number of distinct permutations. These conditions can hold true under specific scenarios. For instance, when the loss is locally (strongly) convex within the $(\delta/2)$-neighborhood under $\Delta_{\rm max}$ of a global solution $\theta_n$, then (stochastic) gradient descent algorithms $\mathcal{A}$ can guarantee convergence to the solution if the initialization $\theta^{(0)}_n$ falls within its $(\delta/2)$-neighborhood. It is also worth noting that when the parameter space is $\Theta={[-B,B]^\mathcal{S}}$, the optimization problem is equivalent when restricted to the effective parameter space $\Theta_0$, as defined in section \ref{sec_cvdeep}. The volume of $\Theta_0$ is $(2B)^\mathcal{S}/(d_1!\cdots d_L!)$. Specifically, when $B$ is fixed, $(2B)^\mathcal{S}/(d_1!\cdots d_L!)$ approaches zero when $d_l\to \infty$ for any $l=1,\ldots,L$. Remarkably, increasing the width of neural networks leads to the effective parameter space's volume tending towards zero. As a result, this may explain the observations in \cite{frankle2018the,allen2019convergence,du2019gradient} where overparameterized networks tend to be easier to train. In \cite{simsek2021geometry}, the geometry (in terms of manifold and connected affine subspace) of sets of minima and critical points in deep learning was described, which also indicates that overparameterized networks bear more minima solutions, thereby facilitating optimization. 

		The landscape of the loss surface in deep learning has been studied by considering the symmetry of the parameter space in several works. 
		Specifically, \cite{brea2019weight} discovered  that permutation critical points are embedded in high-dimensional flat plateaus and proved that all permutation points in a given layer are connected with equal-loss paths. \cite{entezari2021role} conjectured that SGD solutions will likely have no barrier in the linear interpolation between them if the permutation invariance of neural networks is taken into account. \cite{ainsworth2022git} further explored the role of permutation symmetries in the linear mode connectivity of SGD solutions, and argued that neural network loss landscapes often contain (nearly) a single basin after accounting for all possible permutation symmetries of hidden units. Subsequently, \cite{jordan2022repair} proposed methods to mitigate the variance collapse phenomenon that occurs in the interpolated networks, and to improve their empirical performance.  Additionally, optimization algorithms for deep learning have been proposed to enhance training based on the symmetry of network parameter space \citep{badrinarayanan2015understanding, cho2017riemannian, meng2018g, pmlr-v202-navon23a}.

		\begin{remark}
			The popular initialization schemes, including the Xavier and He's methods, use normal random numbers to initialize the entries of weight matrices and bias vectors identically within a layer \citep{glorot2010understanding,he2015delving,shang2016understanding,reddi2019convergence}. By Theorem \ref{opt}, these initializations reduce the optimization difficulty due to the permutation invariance property.
		\end{remark}

		\section{Conclusion}\label{sec_conclusion}
		In this work, we quantitatively characterized the redundancy in the parameterization of deep neural networks based on functional equivalence, and derived a tighter complexity upper bound of the covering number, which is explicit and holds for networks with bias vectors and general activations. We also explored functional equivalence in convolutional, residual and attention-based networks. We discussed the implications for understanding generalization and optimization. Specifically, we found that permutation equivalence can indicate a reduced theoretical complexity of both estimation and optimizations in deep learning.
		
		A limitation of our work is that we only considered permutation invariance, neglecting sign flip and scaling invariance, which may be relevant for specific activations. Furthermore, functional equivalence in practice may be limited to a finite sample, potentially resulting in reduced complexity. Future research could explore the effects of sign flip and scaling invariance and investigate advanced optimization algorithms or designs for deep learning. We also acknowledge that the importance of deriving the lower bound of the covering numbers. We intend to pursue these areas of study to enhance our understanding in the future.

		

	{\small
		\bibliographystyle{apalike}
		\bibliography{EQ_FNN.bib}
	}

	\appendix
	
	\section*{Appendix}
	
	In this Appendix, we present the technical details of the proof of theorems and provide supporting definition and lemmas.
	
	\section{Proof of Theorems}
	In the proofs, we adopt the following notation. We introduce the $L^\infty$-norm of a collection of parameters $\theta=(W^{(1)},b^{(1)},\ldots,W^{(L)},b^{(L)},W^{(L+1)},b^{(L+1)}).$ We define the infinity norm of the collection of parameters by $\Vert\theta\Vert_\infty=\Vert (W^{(1)},b^{(1)},\ldots,W^{(L)},b^{(L)},W^{(L+1)},b^{(L+1)})\Vert_\infty=\max\{\max_{l=1,\ldots,L} \Vert W^{(l)}\Vert_\infty,\max_{l=1,\ldots,L} \Vert b^{(l)}\Vert_\infty\}$. Here, $\Vert \cdot\Vert_\infty$ denotes the maximum absolute value of a vector or matrix.  We let $\Vert\cdot\Vert_2$ denote the $L^2$ norm of a vector. For any matrix $A \in \mathbb{R}^{m \times n}$, the spectral norm of $A$ is denoted by $\Vert A\Vert_2=\max_{x \neq 0} {\Vert Ax\Vert_2}/{\Vert x\Vert_2} $, defined as the largest singular value of $A$ or the square root of the largest eigenvalue of the matrix $A^\top A$. Also we have $\Vert A\Vert_2 \leq \sqrt{mn}\Vert A\Vert_\infty.$
	
	\subsection{Proof of Theorem \ref{thm_perm}}
	\begin{proof}
		The proof is straight forward based on the properties of permutation matrix and element-wise activation functions. Firstly, for any $n\times n$ permutation matrix $P$, it is true that $PP^\top=P^\top P=I_n$ where $I_n$ is the $n\times n$ identity matrix. Secondly, for any element-wise activation function $\sigma$, any $n$-dimensional vector $x\in\mathbb{R}^n$ and any $n\times n$ permutation matrix $P$, it is easy to check
		$$\sigma(Px)=P\sigma(x).$$
		Then for any deep neural network $$f(x;\theta) = W^{(L+1)} \sigma_L(W^{(L)}\cdots\sigma_1(W^{(1)}_1 x+b^{(1)}_1)\cdots)+ b^{(L)}_1) +b^{(L+1)},$$
		and any permutation matrices $P_1,\ldots,P_{L}$, it is easy to check
		\begin{align*}
			&W^{(L+1)} P^\top_L \sigma_L(P_LW^{(L)}P^\top_{L-1}\cdots\sigma_1(P_1W^{(1)} x+P_1b^{(1)})\cdots)+ P_Lb^{(L)}) +b^{(L+1)}\\
			&=W^{(L+1)} \sigma_L(W^{(L)}\cdots\sigma_1(W^{(1)} x+b^{(1)})\cdots)+ b^{(L)}) +b^{(L+1)},
		\end{align*}
		which completes the proof.
	\end{proof}

	\subsection{Proof of Theorem \ref{cv_shallow}}
	
	\begin{proof}	
		
		Firstly, by the property of permutation invariance, we know that the neural networks $\{ f(x;\theta)=W^{(2)}\sigma_1(W^{(1)}x+b^{(1)})+b^{(2)}: \theta\in\Theta_0\}$ parameterized by $\Theta_0$ contains all the functions in $\{ f(x;\theta)=W^{(2)}\sigma_1(W^{(1)}x+b^{(1)})+b^{(2)}: \theta\in\Theta\}$.
		The covering numbers of these two class of functions are the same. Then we can consider the covering number of $\{ f(x;\theta)=W^{(2)}\sigma_1(W^{(1)}x+b^{(1)})+b^{(2)}: \theta\in\Theta_0\}$ where $$\Theta_0:=\{\theta\in[-B,B]^{\mathcal{S}}: b^{(1)}_1\ge b^{(1)}_2\ge\cdots\ge b^{(1)}_{d_1} \}.$$
		
		Recall that for a single hidden layer neural network $f(\cdot;\theta)$ parameterized by $\theta=(W^{(1)},b^{(1)},W^{(2)},b^{(2)})$, the parameters $W^{(1)}\in\mathbb{R}^{d_1\times d_0}$, $b^{(1)}\in\mathbb{R}^{d_1}$, $W^{(2)}\in\mathbb{R}^{1\times d_1}$, and $b^{(2)}\in\mathbb{R}^{1}$ have components bounded by $B$. 
		We start by considering the covering number of the activated linear transformations 
		$$\mathcal{H}:=\{\sigma_1\circ\mathcal{A}_1: (W^{(1)},b^{(1)})\in \Theta_0^{(1)} \},$$
		where $\Theta_0^{(1)}=\{[-B,B]^{d_1\times d_0+d_1}, b^{(1)}_1\ge b^{(1)}_2\ge\cdots\ge b^{(1)}_{d_1}\}$, $\sigma_1$ is a $\rho$-Lipschitz activation function on $[-BB_x-B,BB_x+B]$, and $\mathcal{A}_1(x)=W^{(1)}x+b^{(1)}$. Here $\mathcal{A}_1$ output $d_1$-dimensional vectors, and we define $\Vert\mathcal{A}_1 \Vert_\infty:=\sup_{x\in\mathcal{X}}\Vert \sigma_1\circ\mathcal{A}_1(x) \Vert_2$ for vector-valued functions.
		
		Let $\epsilon_1>0$ be a real number and $\Theta^{(1)}_{0,\epsilon_1}=\{(W^{(1)}_j,b^{(1)}_j)\}_{j=1}^{N_1}$ be a minimal $\epsilon_1$-covering of $\Theta^{(1)}_0$ under $\Vert\cdot\Vert_\infty$ norm with covering number $\mathcal{N}_1=\mathcal{N}(\Theta_0^{(1)},\epsilon_1,\Vert\cdot\Vert_\infty)$. Then for any $(W^{(1)},b^{(1)})\in \Theta_0^{(1)}$, there exist a $(W^{(1)}_j,b^{(1)}_j)$ such that $\Vert (W^{(1)}_j-W^{(1)},b^{(1)}_j-b^{(1)})\Vert_\infty\leq\epsilon_1$. For any $x\in\mathcal{X}$ and $(W^{(1)},b^{(1)})\in \Theta_0^{(1)}$, it is not hard to check $\Vert W^{(1)}x+b^{(1)}\Vert_2\le B(\sqrt{d_0d_1}B_x+1)$. Then we have
		\begin{align*}
			\Vert\sigma_1(W^{(1)}_ix+b^{(1)}_i)-\sigma_1(W^{(1)}_jx+b^{(1)}_j)\Vert_2
			\leq &\rho\Vert(W^{(1)}x+b^{(1)})-(W^{(1)}_jx+b^{(1)}_j)\Vert_2\\
			\leq &\rho\Vert(W^{(1)}-W^{(1)}_j)x\Vert_2+\rho\Vert(b^{(1)}-b^{(1)}_j)\Vert_2\\
			\leq&\rho\sqrt{d_0d_1}\Vert W^{(1)}-W^{(1)}_j\Vert_\infty \Vert x\Vert_2+\rho\sqrt{d_1}\Vert(b^{(1)}-b^{(1)}_j)\Vert_\infty\\
			\leq&\rho\epsilon_1(\sqrt{d_0d_1}B_x+\sqrt{d_1}).
		\end{align*}
		
		This implies that $$\mathcal{H}_1=\{\sigma_1\circ\mathcal{A}_1:(W^{(1)}.b^{(1)})\in \Theta^{(1)}_{0,\epsilon_1}\}$$ is a set with no more than $\mathcal{N}_1$ elements and it covers $\mathcal{H}$ under $\Vert\cdot\Vert_\infty$ norm with radius $\epsilon_1\rho (\sqrt{d_0d_1}B_x+\sqrt{d_1})$.  By Lemma \ref{cv_vol}, the covering number $\mathcal{N}_1=\mathcal{N}(\Theta_0^{(1)},\epsilon,\Vert\cdot\Vert_\infty)\le {\rm Volume(\Theta_0^{(1)})}\times({2}/{\epsilon_1})^{d_1\times d_0+d_1}=(4B/\epsilon_1)^{d_1\times d_0+d_1}/d_1!$. 
		
		Now, let $\epsilon_2>0$ be a real number and $\Theta^{(2)}_{0,\epsilon_2}=\{(W^{(2)}_j,b^{(2)}_j)\}_{j=1}^{N_2}$ be a minimal $\epsilon_2$-covering of $\Theta^{(2)}_0$ under $\Vert\cdot\Vert_\infty$ norm with $\mathcal{N}_2=\mathcal{N}(\Theta_0^{(2)},\epsilon_2,\Vert\cdot\Vert_\infty)\le {\rm Volume(\Theta_0^{(2)})}\times (2/\epsilon_2)^{d_1\times1+1}= (4B/\epsilon_2)^{d_1+1}$. Also we construct a class of functions by
		$$\mathcal{H}_2=\{\mathcal{A}_2\circ h: h\in\mathcal{H}_1,(W^{(2)},b^{(2)})\in\Theta^{(2)}_{0,\epsilon_2}\},$$
		where  $\mathcal{A}_2(x)=W^{(2)}x+b^{(2)}$.  
		Now for any $f=\mathcal{A}_2\circ\sigma_1\circ\mathcal{A}_1$ parameterized by $\theta=(W^{(1)},b^{(1)},W^{(2)},b^{(2)})\in\Theta_0$, by the definition of covering, there exists $h_j\in\mathcal{H}_1$ such that $\Vert h_j-\sigma_1\circ\mathcal{A}_1\Vert_\infty\leq\rho\epsilon_1 (\sqrt{d_0d_1}B_x+\sqrt{d_1})$ and there exists $(W^{(2)}_k,b^{(2)}_k)\in\Theta^{(2)}_{0,\epsilon_2}$ such that $\Vert (W^{(2)}_k-W^{(2)},b^{(2)}_k-b^{(2)})\Vert_\infty\leq\epsilon_2$. Then fro any $x\in\mathcal{X}$
		\begin{align*}
			&\Vert f(x)-W^{(2)}_k h_j(x)+b^{(2)}_k\Vert_2\\
			= & \Vert W^{(2)}\sigma_1\circ\mathcal{A}_1(x)+b^{(2)}-W^{(2)}_k h_j(x)-b^{(2)}_k\Vert_2\\
			\le & \Vert W^{(2)}\sigma_1\circ\mathcal{A}_1(x) -W^{(2)}_k h_j(x)\Vert_2 +\Vert b^{(2)}-b^{(2)}_k\Vert_2\\
			\le & \Vert W^{(2)}\sigma_1\circ\mathcal{A}_1(x) -W^{(2)} h_j(x)\Vert_2 +\Vert W^{(2)} h_j(x) -W^{(2)}_k h_j(x)\Vert_2+\sqrt{d_2}\epsilon_2\\
			\le & \sqrt{d_1}B\rho\epsilon_1(\sqrt{d_0d_1}B_x+\sqrt{d_1})+\epsilon_2 \sqrt{d_1}B(\sqrt{d_0d_1}B_x+\sqrt{d_1}) +\epsilon_2\\
			\le & 2\sqrt{d_0}d_1B(B_x+1)[\rho\epsilon_1+\epsilon_2],
		\end{align*}
		which implies that $\mathcal{H}_2$ is a $(2\sqrt{d_0}d_1B(B_x+1)[\rho\epsilon_1+\epsilon_2])$-covering of the neural networks $\mathcal{F}(1,d_0,d_1,B)$, where there are at most $(4B/\epsilon_1)^{d_1\times d_0+d_1}/d_1!\times (4B/\epsilon_2)^{d_1+1}$ elements in $\mathcal{H}_2$. Given $\epsilon>0$, we take $\epsilon_1=\epsilon/(4\rho \sqrt{d_0}d_1B(B_x+1))$ and $\epsilon_2=\epsilon/(4 \sqrt{d_0}d_1B(B_x+1))$, then this implies
		\begin{align*}
			&\mathcal{N}(\mathcal{F}(1,d_0,d_1,B),\epsilon,\Vert\cdot\Vert_\infty)\\
			\le &(4B/\epsilon_1)^{d_1\times d_0+d_1}/d_1!\times (4B/\epsilon_2)^{d_1+1}\\
			= & (16B/\epsilon)^{d_0\times d_1+d_1+d_1+1}\times (\rho \sqrt{d_0}d_1B(B_x+1))^{d_0\times d_1+d_1}\times (\sqrt{d_0}d_1B(B_x+1))^{d_1+1}/d_1!\\
			=& (16B/\epsilon)^{d_0\times d_1+d_1+d_1+1}\times (\sqrt{d_0}d_1B(B_x+1))^{d_0\times d_1+d_1+d_1+1}\times \rho^{d_0\times d_1+d_1}/d_1!\\
			=& (16B^2(B_x+1)\sqrt{d_0}d_1/\epsilon)^{d_0\times d_1+d_1+d_1+1}\times \rho^{d_0\times d_1+d_1}/d_1!,\\
			=& (16B^2(B_x+1)\sqrt{d_0}d_1/\epsilon)^\mathcal{S} \times\rho^{\mathcal{S}_1}/d_1!,
		\end{align*}
		where $\mathcal{S}=d_0\times d_1+d_1+d_1+1$ is the total number of parameters in the network and $\mathcal{S}_1$ denotes the number of parameters in the linear transformation from the input to the hidden layer.
	\end{proof}

	\subsection{Proof of Theorem \ref{cv_deep}}
	\begin{proof}
		Our proof takes into account permutation equivalence and extends Theorem \ref{cv_shallow} by using mathematical induction. Similar approaches using induction can be found in Lemma A.7 of \cite{bartlett2017spectrally} and Lemma 14 of \cite{lin2019generalization}. However, we improve upon their results in three ways. Firstly, we consider the generally defined neural networks where the bias vector is allowed to appear. Secondly, we express our bound in terms of the width, depth, and size of the neural network as well as the infinity norm of the parameters, instead of the spectral norm of the weight matrices, which can be unknown in practice. Thirdly, we utilize permutation equivalence to derive tighter upper bounds on the covering number.
		
		{\noindent \bf Step 1.}
		We analyze the effective region of the parameter space $$\Theta=[-B,B]^\mathcal{S},$$
		where $\mathcal{S}=\sum_{i=0}^{L} d_id_{i+1}+d_{i+1}$ is the total number of parameters in the deep neural network. We denote $\mathcal{S}_l=d_{l-1}d_{l}+d_l$ by the total number of parameters in $(W^{(l)},B^{(l)})$ in the $l$th layer, and denote $\Theta^{(l)}=\{(W^{(l)},b^{l}): (W^{(l)},b^{(l)})\in[-B,B]^{\mathcal{S}_l}\}$ by the parameter space for $l=1,\ldots,L+1$. By our Theorem \ref{thm_perm}, for any given neural network $f(\cdot;\theta)$ parameterized by $$\theta=(W^{(1)},b^{(1)},\ldots,W^{(L+1)},b^{(L+1)})\in\Theta,$$ 
		there exists $P_1,\ldots,P_{L}$ such that $f(\cdot;\tilde{\theta})$ parameterized by $$\tilde{\theta}=(P_{1}W^{(1)},P_{1}b^{(1)},\ldots,P_{l} W^{(l)} P_{l-1}^\top,P_{l} b^{(1)},\ldots,W^{(L+1)}P_{L}^\top,b^{(L+1)})$$
		implements the same input-output function.  This implies that there exists a  subset $\Theta_0$ of $\Theta$  such that neural networks $\{ f(\cdot;\theta): \theta\in\Theta_0\}$ parameterized by $\Theta_0$ contains all the functions in $\{ f(\cdot;\theta): \theta\in\Theta\}$. The covering numbers of these two class of functions are the same. To be specific, the effective parameter space $$\Theta_0=\Theta_0^{(1)}\times \Theta_0^{(2)}\times\cdots\times\Theta_0^{(L)}\times\Theta_0^{(L+1)},$$
		where
		\begin{align*}
			\Theta_0^{(1)} = &\{(W^{(1)},b^{(1)})\in[-B,B]^{\mathcal{S}_1}: b^{(1)}_1\ge b^{(1)}_2\ge\cdots\ge b^{(1)}_{d_1} \},\\
			\Theta_0^{(l)} = &\{(W^{(l)},b^{(l)})\in[-B,B]^{\mathcal{S}_l}: b^{(l)}_1\ge b^{(l)}_2\ge\cdots\ge b^{(l)}_{d_l} \} {\rm \ for\ }l=2,\ldots,L,\\
			\Theta_0^{(L+1)}& =\{(W^{(L+1)},b^{(L+1)})\in[-B,B]^{\mathcal{S}_{L+1}}\}.\\
		\end{align*}
		In the following, we focus on considering the covering number of $\{ f(\cdot;\theta): \theta\in\Theta_0\}$.
		
		{\noindent \bf Step 2.}
		We start by bounding the covering number for the first activated hidden layer. Let $\mathcal{H}_1=\{\sigma_1\circ\mathcal{A}_1: (W^{(1)}x+b^{(1)})\in\Theta_0^{(1)}\}$ where $\mathcal{A}_1(x)=W^{(1)}x+b^{(1)}$ is the linear transformation from the input to the first hidden layer. Given any $\epsilon_1>0$, in the proof of Theorem \ref{cv_shallow},  we have shown that 
		$$\mathcal{N}_1:=\mathcal{N}(\mathcal{H}_1,\epsilon_1\rho_1\sqrt{d_0d_1}(B_x+1),\Vert\cdot\Vert_\infty)\le{\rm Volume}(\Theta_0^{(1)})\times(2/\epsilon_1)^{\mathcal{S}_1},$$
		and $\Vert h^{(1)}\Vert_\infty\le \rho_1\sqrt{d_0d_1}B(B_x+1)$ for $h^{(1)}\in\mathcal{H}_1.$
		
		{\noindent \bf Step 3.} We use induction to proceed the proof for $l=1,\ldots,L$.
		Let $\mathcal{H}_l=\{\sigma_l\circ\mathcal{A}_l\circ\cdots\circ\sigma_1\circ\mathcal{A}_1: (W^{(k)}x+b^{(k)})\in\Theta_0^{(k)},k=1,\ldots,l\}$ where $\mathcal{A}_k(x)=W^{(k)}x+b^{(k)}$ is the linear transformation from the $(k-1)$th layer to the $k$th layer for $k=1,\ldots,l$. Let $B^{(l)}$ denotes the infinity norm of functions $h^{(l)}\in\mathcal{H}_l$ for $k=1,\ldots,L$. For any $e_l>0$, let $$\tilde{\mathcal{H}}_l=\{h^{(l)}_j\}_{j=1}^{\mathcal{N}(\mathcal{H}_l,e_l,\Vert\cdot\Vert_\infty)}$$
		be a $e_l$-covering of $\mathcal{H}_l$ under the $\Vert\cdot\Vert_\infty$ norm. For any $\epsilon_{l+1}>0$, let
		$$\Theta^{(l+1)}_{0,\epsilon_{l+1}}=\{(W^{(l+1)}_t,b^{(l+1)}_t)\}_{t=1}^{\mathcal{N}(\Theta^{(l+1)}_{0},\epsilon_{l+1},\Vert\cdot\Vert_\infty)}$$
		be a $\epsilon_{l+1}$-covering of $\Theta^{(l+1)}_{0}$. Then for any $h^{(l+1)}=\sigma_{l+1}\circ \mathcal{A}_{l+1}\circ h^{(l)}\in\mathcal{H}_{l+1}$ where $\mathcal{A}_{l+1}(x)=W^{(l+1)}x+b^{(l+1)}$ , there exists $h^{(l)}_j\in \tilde{\mathcal{H}}_l$ and $(W^{(l+1)}_t,b^{(l+1)}_t)\in\Theta^{(l+1)}_{0,\epsilon_{l+1}}$ such that $$\Vert h^{(l)}-h^{(l)}_j\Vert_\infty\le e_l$$ and $$\Vert(W^{(l+1)}_t-W^{(l+1)},b^{(l+1)}_t-b^{(l+1)})\Vert_\infty\le \epsilon_{l+1}.$$
		Then for any $x\in\mathcal{X}$, we have
		\begin{align*}
			&\Vert h^{(l+1)}(x)-\sigma_{l+1}(W^{(l+1)}_t h^{(l)}_j(x)+b^{(l+1)}_t)\Vert_2\\
			= & \Vert \sigma_{l+1}(W^{(l+1)}h^{(l)}(x)+b^{(l+1)})-\sigma_{l+1}(W^{(l+1)}_s h^{(l)}_j(x)-b^{(l+1)}_s)\Vert_2\\
			\le & \rho_{l+1}\Vert (W^{(l+1)}h^{(l)}(x)+b^{(l+1)})-(W^{(l+1)}_s h^{(l)}_j(x)-b^{(l+1)}_s)\Vert_2\\
			\le & \rho_{l+1}\Vert W^{(l+1)}h^{(l)}(x)-W^{(l+1)}_s h^{(l)}_j(x)\Vert_2 +\rho_{l+1}\sqrt{d_{l+1}}\Vert b^{(l+1)}-b^{(l+1)}_s\Vert_\infty\\
			\le & \rho_{l+1}\Vert W^{(l+1)}h^{(l)}(x)-W^{(l+1)} h^{(l)}_j(x)\Vert_2 +\rho_{l+1}\Vert W^{(l+1)}h^{(l)}_j(x)-W^{(l+1)}_s h^{(l)}_j(x)\Vert_2 +\rho_{l+1}\sqrt{d_{l+1}}\epsilon_{l+1}\\
			\le & \rho_{l+1}\sqrt{d_{l}}Be_l +\rho_{l+1}\epsilon_{l+1}\sqrt{d_{l}}B^{(l)} +\rho_{l+1}\sqrt{d_{l+1}}\epsilon_{l+1}\\
			= & \rho_{l+1}(\sqrt{d_{l}}Be_l+(B^{(l)}+\sqrt{d_{l+1}})\epsilon_{l+1}).
		\end{align*}
		Then it is proved that the covering number $\mathcal{N}(\mathcal{H}^{(l+1)},e_{l+1},\Vert\cdot\Vert_\infty)$ with $e_{l+1}=\rho_{l+1}(\sqrt{d_{l}}Be_l+(B^{(l)}+\sqrt{d_{l+1}})\epsilon_{l+1})$ satisfying 
		$$\mathcal{N}(\mathcal{H}^{(l+1)},e_{l+1},\Vert\cdot\Vert_\infty)\le \mathcal{N}(\mathcal{H}^{(l)},e_{l},\Vert\cdot\Vert_\infty)\times\mathcal{N}(\Theta^{(l+1)}_{0},\epsilon_{l+1},\Vert\cdot\Vert_\infty)$$
		for $l=1,\ldots,L-1$. Recall that in the proof of Theorem \ref{cv_shallow}, we have proved that
		$$\mathcal{N}(\mathcal{H}^{(1)},e_1,\Vert\cdot\Vert_\infty)\le\mathcal{N}(\Theta^{(1)}_0,\epsilon_1,\Vert\cdot\Vert_\infty),$$
		where $e_1=\rho_1\epsilon_1\sqrt{d_0d_1}(B_x+1),$ which leads to
		$$\mathcal{N}(\mathcal{H}^{(l)},e_{l},\Vert\cdot\Vert_\infty)\le \Pi_{i=1}^{l}\mathcal{N}(\Theta^{(l)}_{0},\epsilon_{l},\Vert\cdot\Vert_\infty),$$
		for $l=1,\ldots,L$.
		
		{\noindent \bf Step 4.} Next we give upper bounds of $e_l$ and $B^{(l)}$ for $l=1,\ldots,L$. Recall that $B^{(l)}$ denotes the infinity norm of functions $h^{(l)}\in\mathcal{H}_l$ for $k=1,\ldots,L$. Then it is easy to see that $B^{(l)}\ge\rho_lB$ since we can always take the bias vectors to have components $B$ or $-B$. In addition,
		\begin{align*}
			B^{(l+1)}&=\Vert \sigma_{l+1}\circ \mathcal{A}_{l+1}\circ h_{l}\Vert_\infty\le \rho_{l+1}B (\sqrt{d_l}B^{(l)}+\sqrt{d_{l+1}})\le 2\rho_{l+1}\sqrt{d_ld_{l+1}}B B^{(l)}.
		\end{align*}
		As proved in Theorem \ref{cv_shallow}, we know $B^{(1)}\le\rho_1B\sqrt{d_0d_1}(B_x+1)$, then we can get
		$$B^{(l)}\le (2B)^l(B_x+1)\sqrt{d_0}\Pi_{i=1}^{l}d_i\rho_i/(\sqrt{d_l}).$$
		
		Recall $e_{l+1}=\rho_{l+1}(\sqrt{d_{l}}Be_l+(B^{(l)}+\sqrt{d_{l+1}})\epsilon_{l+1})$ for $l=1,\ldots,L-1$, and $e_1=\rho_1\epsilon_1\sqrt{d_0}(B_x+1),$ then by simple mathematics
		\begin{align*}
			e_{L}&=B^{L-1}\Pi_{i=2}^L\rho_i\sqrt{d_{i-1}}e_1+\sum_{i=2}^{L}(B^{(i-1)}+\sqrt{d_i})\epsilon_i\Pi_{j=i}^{L}\rho_j\sqrt{d_{j-1}}/\sqrt{d_{L}}\\
			&\le B^{L-1}\Pi_{i=1}^L\rho_i\sqrt{d_{i-1}} (B_x+1)\epsilon_1 +\sum_{i=2}^L (B_x+1)(2B)^{i-1} \sqrt{d_0}(\Pi_{j=1}^L \rho_j{d_{j}})\epsilon_i/(\sqrt{d_{L}})\\
			& \le 2\sum_{i=1}^L \sqrt{d_0}(B_x+1)(2B)^{i-1}(\Pi_{j=1}^L \rho_j{d_{j}})\epsilon_i/\sqrt{d_{L}}.
		\end{align*}
		
		{\noindent \bf Step 5.} Last we construct a covering of $\mathcal{F}=\mathcal{F}(1,d_0,d_1,B)=\{ f(x;\theta)=W^{(2)}\sigma_1(W^{(1)}x+b^{(1)})+b^{(2)}: \theta\in\Theta_0\}$.
		Let $$\tilde{\mathcal{H}}_L=\{h^{(L)}_j\}_{j=1}^{\mathcal{N}(\mathcal{H}_L,e_L,\Vert\cdot\Vert_\infty)}$$
		be a $e_L$-covering of $\mathcal{H}_L$ under the $\Vert\cdot\Vert_\infty$ norm. For any $\epsilon_{L+1}>0$, let
		$$\Theta^{(L+1)}_{0,\epsilon_{L+1}}=\{(W^{(L+1)}_t,b^{(L+1)}_t)\}_{t=1}^{\mathcal{N}(\Theta^{(L+1)}_{0},\epsilon_{L+1},\Vert\cdot\Vert_\infty)}$$ be a $\epsilon_{L+1}$-covering of $\Theta^{(L+1)}_{0}$. Then for any $f=\mathcal{A}_{L+1}\circ h^{(L)}\in\mathcal{F}$ where $\mathcal{A}_{L+1}(x)=W^{(L+1)}x+b^{(L+1)}$ , there exists $h^{(L)}_j\in \tilde{\mathcal{H}}_L$ and $(W^{(L+1)}_t,b^{(L+1)}_t)\in\Theta^{(L+1)}_{0,\epsilon_{L+1}}$ such that $$\Vert h^{(L)}-h^{(L)}_j\Vert_\infty\le e_L$$ and $$\Vert(W^{(L+1)}_t-W^{(L+1)},b^{(L+1)}_t-b^{(L+1)})\Vert_\infty\le \epsilon_{L+1}.$$
		Then for any $x\in\mathcal{X}$, we have
		\begin{align*}
			&\Vert f(x)-W^{(L+1)}_t h^{(L)}_j(x)+b^{(L+1)}_t\Vert_2\\
			= & \Vert W^{(L+1)}h^{(L)}(x)+b^{(L+1)}-W^{(L+1)}_s h^{(L)}_j(x)-b^{(L+1)}_s\Vert_2\\
			\le & \Vert (W^{(L+1)}h^{(L)}(x)+b^{(L+1)})-(W^{(L+1)}_s h^{(L)}_j(x)-b^{(L+1)}_s)\Vert_2\\
			\le & \Vert W^{(L+1)}h^{(L)}(x)-W^{(L+1)}_s h^{(L)}_j(x)\Vert_2 +\Vert b^{(L+1)}-b^{(L+1)}_s\Vert_2\\
			\le & \Vert W^{(L+1)}h^{(L)}(x)-W^{(L+1)} h^{(L)}_j(x)\Vert_2 +\Vert W^{(L+1)}h^{(L)}_j(x)-W^{(L+1)}_s h^{(L)}_j(x)\Vert_2 +\epsilon_{L+1}\\
			\le & \sqrt{d_{L}}Be_L +\epsilon_{L+1}\sqrt{d_{L}}B^{(L)} +\epsilon_{L+1}\\
			= & (\sqrt{d_{L}}Be_L+(\sqrt{d_L}B^{(L)}+1)\epsilon_{L+1})\\
			\le &  2\sqrt{d_0}\sum_{i=1}^L(B_x+1) (2B)^{i}(\Pi_{j=1}^L \rho_jd_{j})\epsilon_i+ 2(2B)^L(B_x+1)\sqrt{d_0}(\Pi_{j=1}^L \rho_j{d_{j}}) \epsilon_{L+1}\\
			\le &2\sqrt{d_0}(B_x+1)(\Pi_{j=1}^L \rho_j{d_{j}})\sum_{i=1}^{L+1} (2B)^{i}\epsilon_i.
		\end{align*}
		Then we know that the covering number $\mathcal{N}(\mathcal{F},e_{L+1},\Vert\cdot\Vert_\infty)$ with radius $$e_{L+1}=2(B_x+1)\sqrt{d_0}(\Pi_{j=1}^L \rho_j{d_{j}})\sum_{i=1}^{L+1} (2B)^{i-1}\epsilon_i$$ satisfies 
		\begin{align*}
			\mathcal{N}(\mathcal{F},e_{L+1},\Vert\cdot\Vert_\infty)&\le \mathcal{N}(\mathcal{H}^{(L)},e_{L},\Vert\cdot\Vert_\infty)\times\mathcal{N}(\Theta^{(L+1)}_{0},\epsilon_{L+1},\Vert\cdot\Vert_\infty)\\
			&\le \Pi_{i=1}^{L+1}\mathcal{N}(\Theta^{(i)}_{0},\epsilon_{L+1},\Vert\cdot\Vert_\infty)\\
			&\le \Pi_{i=1}^{L+1} {\rm Volume}(\Theta^{(i)}_{0})\times (2/\epsilon_i)^{\mathcal{S}_i}\\
			&= (4B/\epsilon_{L+1})^{\mathcal{S}_{L+1}} \times\Pi_{i=1}^{L} (4B/\epsilon_i)^{\mathcal{S}_i}/(d_i!).
		\end{align*}
		Finally, setting $\epsilon_i=\{2(L+1)\sqrt{d_0}(B_x+1)(2B)^{i}(\Pi_{j=1}^L \rho_j{d_{j}})\}^{-1}\epsilon$ for $i=1,\ldots,L+1$ leads to an upper bound for the covering number $\mathcal{N}(\mathcal{F},\epsilon,\Vert\cdot\Vert_\infty)$
		\begin{align*}
			\mathcal{N}(\mathcal{F},\epsilon,\Vert\cdot\Vert_\infty)
			&\le  (4(L+1)(B_x+1)\sqrt{d_0}(B_x+1)(2B)^{L+2}(\Pi_{j=1}^L \rho_j{d_{j}})/\epsilon)^{\mathcal{S}_{L+1}} \\
			&\qquad\times\Pi_{i=1}^{L} (4(L+1)(2B)^{i+1}(\Pi_{j=1}^L \rho_j\sqrt{d_{j}})/\epsilon)^{\mathcal{S}_i}/(d_i!)\\
			&\le \frac{(4(L+1)\sqrt{d_0}(B_x+1)(2B)^{L+2}(\Pi_{j=1}^L \rho_j{d_{j}})/\epsilon)^{\mathcal{S}}}{d_1!\times d_2!\times\cdots\times d_L!}.
		\end{align*}
		This completes the proof.
	\end{proof}

	\subsection{Proof of Theorem \ref{opt}}
	\begin{proof}
		The proof is straightforward. And we present the proof in three steps. First, we show that if $f_{\theta_{n}}$ is an empirical risk minimizer, then there are at least $d_1^*\times d_2^*\times \cdots\times d_L^*$ empirical risk minimizers with distinct parameterization. Second, we prove that the $(\delta/2)$ neighborhood of these distinct parameterization are disjoint under the $L_\infty$ norm. Lastly, we show that initialization schemes with identical random distribution for weights and bias within layers can indeed increase the probability of convergence for any appropriate optimization algorithms $\mathcal{A}$.
		
		{\noindent \bf Step 1.}
		Suppose $f_{\theta_{n}}$ is an empirical risk minimizer with parameterization
		\begin{align*}
			\theta_{n}=(W^{(1)}_n,b^{(1)}_n,W^{(2)}_n,b^{(2)}_n,\ldots,W^{(L)}_n,b^{(L)}_n,W^{(L+1)}_n,b^{(L+1)}_n).
		\end{align*}
		By Theorem \ref{thm_perm}, for any permutation matrices $P_1,\ldots,P_L$,
		
		{\noindent \bf Step 1.}
		Suppose $f_{\theta_{n}}$ is an empirical risk minimizer with parameterization
		\begin{align*}
			\theta_{n}=(W^{(1)}_n,b^{(1)}_n,W^{(2)}_n,b^{(2)}_n,\ldots,W^{(L)}_n,b^{(L)}_n,W^{(L+1)}_n,b^{(L+1)}_n).
		\end{align*}
		By Theorem \ref{thm_perm}, for any permutation matrices $P_1,\ldots,P_L$,
		\begin{equation}\label{permuted}
			\tilde{\theta}_{n}=(P_1W^{(1)}_n,P_1b^{(1)}_n,P_2W^{(2)}_nP_1^\top,P_2b^{(2)}_n,\ldots,P_LW^{(L)}_n,P_Lb^{(L)}_n,W^{(L+1)}_nP_L^\top,b^{(L+1)}_n),
		\end{equation}
		will lead to an empirical risk minimizer $f_{\tilde{\theta}_n}$.
		However, the concatenated matrices $(W^{(l)}_n;b^{(l)}_n)$ may have identical rows, and the permutation-implemented matrices $(P_lW^{(l)}_nP_{l-1}^\top; P_l;b^{(l)}_n)$ may remain unchanged for some permutation matrices $P_l$. Let $d^*_l$ denote the number of distinct permutations of the rows in $(W^{(l)}_n;b^{(l)}_n)$. Then it is guaranteed that $\{(P_lW^{(l)}_nP_{l-1}^\top; P_l;b^{(l)}_n): P_l, P_{l-1}{\rm \ are\ permutation\ matrices}\}$ has at least  $d^*_l$ distinct elements. Note that $1\le d^*_l\le d_l!$ for $l=1,\ldots,L$ where $d_l$ is the dimension of the bias vector $b^{(l)}$ as well as the number of rows of $(W^{(l)}_n;b^{(l)}_n)$. Specifically, $d^*_l=1$ if and only if all the entries of $b^{(l)}_n$ are identical and all the rows of $W^{(l)}$ are identical. And $d^*_l=d_l!$ if and only if all the rows of concatenated matrix $(W^{(l)}_n,b^{(l)}_n)$ are distinct. Moreover, the distinct elements in $\{(P_lW^{(l)}_nP_{l-1}^\top; P_l;b^{(l)}_n): P_l, P_{l-1}{\rm \ are\ permutation\ matrices}\}$ can range from $1$ to $d_{l-1}!d_l!$. It is 1 if and only if all the entries of $b^{(l)}_n$ are identical and all the entries of $W^{(l)}_n$ are identical;  it is $d_{l-1}!d_l!$ if and only if all the entries of $b^{(l)}_n$ are distinct and all the entries of $W^{(l)}_n$ are distinct.  Lastly, it is easy to see that $\theta_n\not=\tilde{\theta}_n$ if $(W^{(l)}_n;b^{(l)}_n)\not=(P_lW^{(l)}_nP_{l-1}^\top; P_l;b^{(l)}_n)$ for any $l\in\{1,\ldots,L\}$. Then there are at least $d_1^*\times\cdots\times d_L^*$ distinct elements in
		\begin{align}\label{permuated_distinct}
			\tilde{\Theta}_n=\{\tilde{\theta}_{n}{\rm \ defined\ in\ (\ref{permuted})}: P_1,\ldots P_{L}{\rm \ are\ permutation\ matrices}\}.
		\end{align}
		
		{\noindent \bf Step 2.} 
		
		Let $\theta=(W^{(1)},b^{(1)},W^{(2)},b^{(2)},\ldots,W^{(L)},b^{(L)},W^{(L+1)},b^{(L+1)})$ be the collection of parameters of a network and let $\theta_i^{(l)}$ be the $i$th row of concatenated matrix $\theta^{(l)}=(W^{(l)};b^{(l)})$. We define
		\begin{equation}\label{delta_min}
			\Delta_{\rm min}(\theta):=\min_{l\in\{1,\ldots,L\}}\left[\min_{i,j\in\{1,\ldots,d_l\},\theta^{(l)}_i\not=\theta^{(l)}_j}\Vert \theta^{(l)}_i-\theta^{(l)}_j\Vert_\infty\right],
		\end{equation}
		and
		\begin{equation}\label{delta_max}
			\Delta_{\rm max}(\theta):=\max_{l\in\{1,\ldots,L\}}\left[\max_{i,j\in\{1,\ldots,d_l\}}\Vert \theta^{(l)}_i-\theta^{(l)}_j\Vert_\infty\right].
		\end{equation}
		
		Recall that $\Delta_{\rm min}(\theta_{n})=\delta$. 
		This implies that for any two distinct permutation-implemented $\tilde{\theta}_{n,1}$ and $\tilde{\theta}_{n,2}$ in (\ref{permuated_distinct}),
		$$\Delta_{\rm max}(\tilde{\theta}_{n,1}-\tilde{\theta}_{n,2})\ge \delta,$$
		where $\Delta_{\rm max}(\theta_n)$ denote the maximum of the $\Vert\cdot\Vert_\infty$ norm of distinct rows in $\theta^{(l)}_n$ over $l\in{1,\ldots,L}$. Then the neighborhoods  $B_\infty(\tilde{\theta}_n,\delta/2):=\{\theta:\Delta_{\rm max}(\theta-\tilde{\theta}_n)\le\delta/2\}$ in the collection
		\begin{align*}
			\mathcal{B}_{(\delta/2)}=\{B_\infty(\tilde{\theta}_n,\delta/2): \tilde{\theta}_n\in\tilde{\Theta}_n \}
		\end{align*}
		are pairwise disjoint. It worth mentioning that for $\tilde{\theta}=\mathcal{P}(\theta_n)$, we have
		\begin{align}\label{perm_infty}
			B_\infty(\tilde{\theta}_n,\delta/2)=\{\mathcal{P}(\theta):\Delta_{\rm max}(\theta-\theta_{n})\le\delta/2\},
		\end{align}
		by the symmetry of permutation and the definition of $\Delta_{\rm max}$.

		{\noindent \bf Step 3.} For any given permutation matrices $P_1,\ldots,P_L$, we let $\mathcal{P}=\mathcal{P}(P_1,\ldots,P_L)$ denote the operator such that 
		$$\mathcal{P}(\theta)=\tilde{\theta}$$
		for any $\theta\in\Theta$ where
		\begin{align*}
			\theta_{n}=(W^{(1)}_n,b^{(1)}_n,W^{(2)}_n,b^{(2)}_n,\ldots,W^{(L)}_n,b^{(L)}_n,W^{(L+1)}_n,b^{(L+1)}_n).
		\end{align*}
		and 
		\begin{align*}
			\tilde{\theta}_{n}=(P_1W^{(1)}_n,P_1b^{(1)}_n,P_2W^{(2)}_nP_1^\top,P_2b^{(2)}_n,\ldots,P_LW^{(L)}_n,P_Lb^{(L)}_n,W^{(L+1)}_nP_L^\top,b^{(L+1)}_n).
		\end{align*}
		Now if an optimization algorithm $\mathcal{A}$ guarantees to produce a convergent solution towards $\theta_{n}$ when $\theta^{(0)}$ is an initialization belong to the $(\delta/2)$-neighborhood of $\theta_{n}$, then for any $\mathcal{P}$, the algorithm $\mathcal{A}$ should guarantee to produce a convergent solution towards $\mathcal{P}(\theta_{n})$ when $\theta^{(0)}$ is an initialization belong to the $(\delta/2)$-neighborhood of $\mathcal{P}(\theta_{n})$. The reason is that the loss surface of the empirical risk keeps the same structure on $$B_\infty({\theta}_n,\delta/2)=\{\theta^{(0)}:\Delta_{\rm max}(\theta^{(0)}-\theta_{n})\le\delta/2\}$$
		and 
		$$B_\infty(\tilde{\theta}_n,\delta/2)=\{\mathcal{P}(\theta^{(0)}):\Delta_{\rm max}(\theta^{(0)}-\theta_{n})\le\delta/2\}.$$ 
		This implies that if the initialization $\theta^{(0)}$ belongs to $(\delta/2)$-neighborhood of any $\tilde{\theta}_n\in\tilde{\Theta}_n$, the algorithm $\mathcal{A}$ should guarantee to produce a convergent solution towards some $\tilde{\theta}_n\in\tilde{\Theta}_n$, which learns the an empirical risk minimizer $f_{\tilde{\theta}_{n}}$. The rest of the proof is to calculate the probability of an random initialization $\theta^{(0)}$ locates in the union of neighborhoods in $\mathcal{B}_{\delta/2}$, i.e, we targets for 
		$$\mathbb{P}(\theta^{(0)}\in\cup_{\tilde{\theta}_n\in\tilde{\Theta}_n} B_\infty(\tilde{\theta}_n,\delta/2)),$$
		where the probability is with respect to the randomness of initialization. Firstly, for any random initialization scheme and permutation operator $\mathcal{P}$, we have
		\begin{align}\label{prob1}
			\mathbb{P}(\theta^{(0)}\in B_\infty({\theta}_n,\delta/2))=\mathbb{P}(\mathcal{P}(\theta^{(0)})\in B_\infty(\mathcal{P}(\theta_n),\delta/2)),
		\end{align}
		by (\ref{perm_infty}) and the definition of permutation. With a little bit abuse of notation, let 
		$$\theta^{(0)}=(W^{(1)}_{(0)}, b^{(1)}_{(0)},\ldots, W^{(L+1)}_{(0)}, b^{(L+1)}_{(0)}).$$	
		If the initialization method uses identical random distributions for the entries of weights and biases within a layer, then we know that the entries of $W^{(l)}_{(0)}$ are independent and identically distributed, and entries of $b^{(l)}_{(0)}$ are independent and identically distributed. This further implies that for any permutation matrices $P_1,\ldots,P_L$,
		\begin{gather*}
			(W^{(1)}_{(0)},b^{(1)}_{(0)})\stackrel{d}{=}(P_1W^{(1)}_{(0)},P_1b^{(1)}_{(0)}),\\
			(W^{(l)}_{(0)},b^{(l)}_{(0)})\stackrel{d}{=}(P_lW^{(l)}_{(0)}P^\top_{l-1},P_lb^{(l)}_{(0)}),\quad  l=2,\ldots,L\\
			(W^{(L+1)}_{(0)},b^{(L+1)}_{(0)})\stackrel{d}{=}(W^{(L+1)}_{(0)}P^\top_{L},b^{(L+1)}_{(0)}),
		\end{gather*}
		where $\stackrel{d}{=}$ denotes the equivalence in distribution. Consequently, $\mathcal{P}(\theta^{(0)})$ has the same distribution of $\theta^{(0)}$. And we have
		\begin{align}\label{prob2}
			\mathbb{P}(\mathcal{P}(\theta^{(0)})\in B_\infty(\mathcal{P}(\theta_n),\delta/2))=\mathbb{P}(\theta^{(0)}\in B_\infty(\mathcal{P}(\theta_n),\delta/2)),
		\end{align}
		for any permutation $\mathcal{P}$.
		
		Combining (\ref{prob1}), (\ref{prob2}) and the property in (\ref{permuated_distinct}), we have
		\begin{align*}
			\mathbb{P}(\theta^{(0)}\in\cup_{\tilde{\theta}_n\in\tilde{\Theta}_n} B_\infty(\tilde{\theta}_n,\delta/2))&=\Sigma_{\tilde{\theta}\in\tilde{\Theta}} \mathbb{P}(\theta^{(0)}\in B_\infty(\tilde{\theta}_n,\delta/2))\\
			&=\Sigma_{\tilde{\theta}\in\tilde{\Theta}} \mathbb{P}(\theta^{(0)}\in B_\infty({\theta}_n,\delta/2))\\
			&=d_1^*\times\cdots\times d_L^* \times\mathbb{P}(\theta^{(0)}\in B_\infty({\theta}_n,\delta/2))\\
			&=d_1^*\times\cdots\times d_L^* \times\mathbb{P}(\Delta_{\rm max}(\theta^{(0)}-\theta_n)\le \delta/2).
		\end{align*}
		This completes the proof.

	\end{proof}

	\section{Supporting Lemmas}
	In this section, we give definitions of covering number, packing number of subsets in Euclidean space. We also present definitions of other complexity measures including VC-dimension, Pseudo-dimension, and Rademacher complexity. Lemmas regarding their properties and correlations are also given.
	\begin{definition}[Covering number]
		Let $(K,\Vert\cdot\Vert)$ be a metric space, let $C$ be a subset of $K$, and let $\epsilon$ be a positive real number. Let $B_\epsilon(x)$ denote the ball of radius $\epsilon$ centered at $x$. Then $C$ is called a $\epsilon$-covering of $K$, if $K\subset\cup_{x\in C}B_\epsilon(x).$ The covering number of the metric space $(K,\Vert\cdot\Vert)$ with any radius $\epsilon>0$ is the minimum cardinality of any $\epsilon$-covering, which is defined by $\mathcal{N}(K,\epsilon,\Vert\cdot\Vert)=\min\{\vert C\vert:C {\rm\ is\ a\ } \epsilon{\rm -covering\ of\ } K\}$.
	\end{definition}
	
	\begin{definition}[Packing number]
		Let $(K,\Vert\cdot\Vert)$ be a metric space, let $P$ be a subset of $K$, and let $\epsilon$ be a positive real number. Let $B_\epsilon(x)$ denote the ball of radius $\epsilon$ centered at $x$. Then $P$ is called a $\epsilon$-packing of $K$, if $\{B_\epsilon(x)\}_{x\in P}$ is pairwise disjoint. The $\epsilon$-packing number of the metric space $(K,\Vert\cdot\Vert)$ with any radius $\epsilon>0$ is the maximum cardinality of any $\epsilon$-packing, which is defined by $\mathcal{M}(K,\epsilon,\Vert\cdot\Vert)=\max\{\vert P\vert:P {\rm\ is\ a\ } \epsilon{\rm -packing\ of\ } K\}$.
	\end{definition}
	
	\begin{lemma}\label{cv_pk}
		Let $(K,\Vert\cdot\Vert)$ be a metric space, and for any $\epsilon>0$, let $\mathcal{N}(K,\epsilon,\Vert\cdot\Vert)$ and $\mathcal{M}(K,\epsilon,\Vert\cdot\Vert)$ denote the $\epsilon$-covering number and $\epsilon$-packing number respectively, then
		$$\mathcal{M}(K,2\epsilon,\Vert\cdot\Vert)\leq\mathcal{N}(K,\epsilon,\Vert\cdot\Vert)\leq\mathcal{M}(K,\epsilon/2,\Vert\cdot\Vert).$$
		\begin{proof}
			For simplicity, we write $\mathcal{N}_\epsilon=\mathcal{N}(K,\epsilon,\Vert\cdot\Vert)$ and $\mathcal{M}_\epsilon=\mathcal{M}(K,\epsilon,\Vert\cdot\Vert)$.
			We firstly proof $\mathcal{M}_{2\epsilon}\leq\mathcal{N}_\epsilon$  by contradiction. Let $P=\{p_1,\ldots,p_{\mathcal{M}_{2\epsilon}}\}$ be any maximal $2\epsilon$-packing of $K$ and $C=\{c_1,\ldots,c_{\mathcal{N}_{\epsilon}}\}$ be any minimal $\epsilon$-covering of $K$. If $\mathcal{M}_{2\epsilon}\geq\mathcal{N}_\epsilon+1$, then we must have $p_i$ and $p_j$ belonging
			to the same $\epsilon$-ball $B_\epsilon(c_k)$ for some $i\not= j$ and $k$. This means that the distance between $p_i$ and $p_j$ cannot be more than the diameter of the ball, i.e. $\Vert p_i-p_j\Vert\leq2\epsilon$, which leads to a contradiction since $\Vert p_i-p_j\Vert>2\epsilon$ by the definition of packing.
			
			Secondly, we prove $\mathcal{N}_{\epsilon}\leq\mathcal{M}_{\epsilon/2}$ by showing that each maximal $(\epsilon/2)$-packing $P=\{p_1,\ldots,p_{\mathcal{M}_\epsilon}\}$ is also a $\epsilon$-covering. Note that for any $x\in K\backslash P$, there exist a $p_i\in P$ such that $\Vert x-p_i\Vert\leq\epsilon$ since if this does not hold, then we can construct a bigger packing with
			$p_{\mathcal{M}_\epsilon+1}=x$. Thus $P$ is also a  $\epsilon$-covering and we have $\mathcal{N}_{\epsilon}\leq\mathcal{M}_\epsilon$ by the definition of covering.
		\end{proof}	
	\end{lemma}

	\begin{lemma}\label{cv_vol}
		Let $S$ be a subset of $\mathbb{R}^d$ with volume $V$, and let $\epsilon>0$ be a positive real number. Then, the covering number $\mathcal{N}(V,\epsilon,\Vert\cdot\Vert_\infty)$  and packing number $\mathcal{M}(V,\epsilon/2,\Vert\cdot\Vert_\infty)$ of $S$ satisfies:
		$$\mathcal{N}(V,\epsilon,\Vert\cdot\Vert_\infty)\le\mathcal{M}(V,\epsilon/2,\Vert\cdot\Vert_\infty)\leq V\times(\frac{2}{\epsilon})^d.$$
	\end{lemma}
	
	\begin{proof}
		Consider a packing of $S$ with non-overlapping hypercubes of side length $\epsilon/2$ under the $L^\infty$ norm. The volume of each hypercube is $(\epsilon/2)^d$, and since the hypercubes do not overlap, the total volume of the hypercubes in the packing is at most the volume of $S$. Thus, we have:
		$$\mathcal{M}(V,\epsilon/2,\Vert\cdot\Vert_\infty)\cdot(\epsilon/2)^d\leq V,$$
		which implies:
		$$\mathcal{M}(V,\epsilon/2,\Vert\cdot\Vert_\infty)\leq V\times(\frac{2}{\epsilon})^d.$$
		Then by Lemma \ref{cv_pk}, $\mathcal{N}(V,\epsilon,\Vert\cdot\Vert_\infty)\le\mathcal{M}(V,\epsilon/2,\Vert\cdot\Vert_\infty)$, which completes the proof. 
		
	\end{proof}

	\begin{definition}[Shattering, Definition 11.4 in \cite{mohri2018foundations}]
		Let $\mathcal{F}$ be a family of functions from a set $\mathcal{Z}$ to $\mathbb{R}$. A set $\{z_1,\ldots,Z_n\}\subset\mathcal{Z}$ is said to be shattered by $\mathcal{F}$, if there exists $t_1,\ldots,t_n\in\mathbb{R}$ such that
		\begin{align*}
			\Bigg\vert\Big\{\Big[
			\begin{array}{lr}
				{\rm sgn}(f(z_1)-t_1)\\
				\ldots\\
				{\rm sgn}(f(z_n)-t_n)\\
			\end{array}\Big]:f\in\mathcal{F}
			\Big\}\Bigg\vert=2^n,
		\end{align*}
		where ${rm sgn}$ is the sign function returns $+1$ or $-1$ and $\vert\cdot\vert$ denotes the cardinality of a set. When they exist, the threshold values $t_1,\ldots,t_n$ are said to witness the shattering.
	\end{definition}

	\begin{definition}[Pseudo dimension, Definition 11.5 in \cite{mohri2018foundations}]
		Let $\mathcal{F}$ be a family of functions mapping from $\mathcal{Z}$ to $\mathbb{R}$. Then, the pseudo dimension of $\mathcal{F}$, denoted by ${\rm Pdim}(\mathcal{F})$, is the size of the largest set shattered by $\mathcal{F}$.
	\end{definition}
	
	\begin{definition}[VC dimension]
		Let $\mathcal{F}$ be a family of functions mapping from $\mathcal{Z}$ to $\mathbb{R}$. Then, the Vapnik–Chervonenkis (VC) dimension of $\mathcal{F}$, denoted by ${\rm VCdim}(\mathcal{F})$, is the size of the largest set shattered by $\mathcal{F}$ with all threshold values being zero, i.e., $t_1=\ldots,=t_n=0$.
	\end{definition}

	\begin{definition}[Empirical Rademacher Complexity, Definition 3.1 in \cite{mohri2018foundations}]
		Let $\mathcal{F}$ be a family of functions mapping from $\mathcal{Z}$ to $[a,b]$ and $S=(z_1,\ldots,z_n)$ a fixed sample of size m with elements in $\mathcal{Z}$. Then, the empirical Rademacher complexity of $\mathcal{F}$ with respect to the sample $S$ is defined as:
		\begin{equation*}
			\hat{\mathcal{R}}_S (\mathcal{F}) = \mathbb{E}_{\bf \sigma} \Bigg[ \sup_{f \in \mathcal{F}} \Bigg| \frac{1}{n} \sum_{i=1}^n \sigma_i f(x_i) \Bigg| \Bigg],
		\end{equation*}
		where ${\bf \sigma}=(\sigma_1,\ldots,\sigma_n)^\top$, with $\sigma_i$s independent uniform random variables taking values in $\{+1,-1\}$. The random variables $\sigma_i$ are called Rademacher variables.
	\end{definition}

	\begin{definition}[Rademacher Complexity, Definition 3.2 in \citet{mohri2018foundations}]
		Let $\mathcal{D}$ denote the distribution according to which samples are drawn. For any integer $n\ge1$, the Rademacher complexity of $\mathcal{F}$ is the expectation of the empirical Rademacher complexity over all samples of size $n$ drawn according to $\mathcal{D}$:
		\begin{equation*}
			\mathcal{R}_n(\mathcal{F}) =\mathbb{E}_{S\sim\mathcal{D}^n}\left[\hat{\mathcal{R}}_S(\mathcal{F})\right].
		\end{equation*}
	\end{definition}
	
	\begin{lemma}[Dudley’s Theorem, \cite{dudley1967sizes}] Let $\mathcal{F}$ be a set of functions $f:\mathcal{X}\to\mathbb{R}$. Then
		\begin{align*}
			{\mathcal{R}}_n(\mathcal{F}) \le 12\int_0^\infty \sqrt{\frac{\log\mathcal{N}(\mathcal{F},\epsilon,\Vert\cdot\Vert_\infty)}{n}}d\epsilon.
		\end{align*}
	\end{lemma}
	Dudley’s Theorem gives a way to bound Rademacher complexity using covering number \citep{dudley1967sizes,dudley2010universal}.  And the covering and packing numbers can be bounded in terms of the VC dimension or Pseudo dimension \citep{haussler1995sphere,anthony1999neural}.

	\begin{lemma}[Theorem 12.2 in \cite{anthony1999neural}]
		Suppose that $\mathcal{F}$ is a class of functions from $X$ to the bounded interval $[0,B]\subset\mathbb{R}$. Given a sequence $x=(x_1,\ldots,x_n)\in\mathcal{X}^n$, we let  $\mathcal{F}|_x$ be the subset of $\mathbb{R}^n$ given by
		$$\mathcal{F}|_x=\{(f(x_1),\ldots,f(x_n)):f\in\mathcal{F}\}.$$
		we define the uniform covering number $\mathcal{N}_n(\mathcal{F},\epsilon,\Vert\cdot\Vert_\infty)$  to be the maximum, over all $x\in\mathcal{X}^n$, of the covering number $\mathcal{N}(\mathcal{F}|_x,\epsilon,\Vert\cdot\Vert_\infty)$, that is
		$$\mathcal{N}_n(\mathcal{F},\epsilon,\Vert\cdot\Vert_\infty)=\max\{\mathcal{N}(\mathcal{F}|_x,\epsilon,\Vert\cdot\Vert_\infty): x\in\mathcal{X}^n\}.$$
		Let $\epsilon>0$ and suppose the pseudo-dimension of $\mathcal{F}$ is $d$. Then
		\begin{align*}
			\mathcal{N}_n(\mathcal{F},\epsilon,\Vert\cdot\Vert_\infty)\le \sum_{i=1}^d {\binom{n}{i}}(\frac{B}{\epsilon})^i,
		\end{align*}
		which is less that $(enB/(\epsilon d))^d$ for $n\ge d$.
	\end{lemma}

\end{document}